\pdfoutput=1
\documentclass[10pt]{article}

\usepackage{arxiv}
\usepackage[nonamebreak]{natbib}
\setcitestyle{authoryear,round,citesep={;},aysep={,},yysep={;}}

\usepackage[utf8]{inputenc} 
\usepackage[T1]{fontenc}    
\usepackage{hyperref}       
\usepackage{url}            
\usepackage{booktabs}       
\usepackage{multirow}
\usepackage{amsfonts}       
\usepackage{nicefrac}       
\usepackage{microtype}      
\usepackage{graphicx}
\usepackage{subcaption}
\usepackage{amsmath}
\usepackage{amssymb}
\usepackage{xcolor}
\usepackage{algorithm}
\usepackage[noend]{algpseudocode}
\usepackage{amsthm}
\usepackage{afterpage}
\theoremstyle{plain}
\newtheorem{theorem}{Theorem}

\newtheorem{lemma}{Lemma}
\theoremstyle{definition}

\newtheorem{corollary}{Corollary}
\theoremstyle{remark}
\newtheorem{remark}{Remark}

\title{Trust, Don't Trust, or Flip: Robust Preference-Based Reinforcement Learning with Multi-Expert Feedback}

\author{
  Seyed Amir Hosseini\thanks{Equal contribution.} \\
  {\normalfont\small Sharif University of Technology} \\
  {\normalfont\small Tehran, Iran} \\
  {\normalfont\small \texttt{sa.hosseini.00@sharif.edu}}
  \And
  Maryam Abdolali\footnotemark[1]\;\thanks{Corresponding author.} \\
  {\normalfont\small K. N. Toosi University of Technology} \\
  {\normalfont\small Tehran, Iran} \\
  {\normalfont\small \texttt{maryam.abdolali@kntu.ac.ir}}
  \And
  Fardin Ayar \\
  {\normalfont\small Amirkabir University of Technology} \\
  {\normalfont\small Tehran, Iran} \\
  {\normalfont\small \texttt{fardin.ayar@aut.ac.ir}}
  \AND
  Amirhosein Tavakkoli \\
  {\normalfont\small K. N. Toosi University of Technology} \\
  {\normalfont\small Tehran, Iran} \\
  {\normalfont\small \texttt{amirhosein.tavakkoli@email.kntu.ac.ir}}
  \And
  Ehsan Javanmardi \\
  {\normalfont\small The University of Tokyo} \\
  {\normalfont\small Tokyo, Japan} \\
  {\normalfont\small \texttt{ejavanmardi@g.ecc.u-tokyo.ac.jp}}
  \AND
  Manabu Tsukada \\
  {\normalfont\small The University of Tokyo} \\
  {\normalfont\small Tokyo, Japan} \\
  {\normalfont\small \texttt{mtsukada@g.ecc.u-tokyo.ac.jp}}
  \And
  Mahdi Javanmardi \\
  {\normalfont\small Amirkabir University of Technology} \\
  {\normalfont\small Tehran, Iran} \\
  {\normalfont\small \texttt{mjavan@aut.ac.ir}}
}

\begin{document}
\maketitle
\begin{abstract}
Preference-based reinforcement learning (PBRL) learns reward functions from pairwise trajectory comparisons, offering an alternative to explicit reward engineering. In practice, preference data often come from heterogeneous annotators: some accurate, some noisy, and some systematically adversarial. Existing PBRL methods either weight all feedback equally or filter unreliable samples; both strategies struggle when annotators are anti-aligned rather than merely noisy, because filtering discards signal that could instead be recovered by inversion. We introduce \textbf{TriTrust-PBRL (TTP)}, which jointly learns a shared reward model and expert-specific trust parameters. During training, trust values move toward positive (trust), near-zero (down-weight), or negative (invert), so that systematically adversarial experts contribute usable signal rather than being discarded. We support the method with a gradient analysis of trust separation and an identifiability result for the reward and trust parameters up to affine reparameterization, and we make explicit the residual global sign ambiguity together with the reliable-majority condition under which our stabilized optimization selects the correct sign. Controlled synthetic experiments map out when this mechanism succeeds, and end-to-end PEBBLE experiments on four continuous-control domains (MetaWorld and DMControl) show that, under a reliable majority of adversarial or noisy experts, TTP recovers strong policies where uniform-weight PEBBLE collapses and improves over sample-level robust baselines in the adversarial setting. TTP requires no expert features beyond an identity index and integrates into existing PBRL pipelines as a drop-in reward-learning objective.
\end{abstract}


\section{Introduction}
\label{sec:introduction}

Preference-based reinforcement learning (PBRL) replaces a hand-designed reward with a reward model learned from pairwise comparisons of agent behavior~\citep{christiano2017deep,wirth2017survey,pebble}. This formulation has enabled agents to learn from judgments that are often easier to provide than calibrated scalar rewards. Many standard PBRL methods, however, assume that preference labels are generated by a common feedback process and treat them uniformly~\citep{christiano2017deep,pebble,lee2021bpref}. That assumption is restrictive when data are collected from multiple experts whose judgments differ systematically.

Such heterogeneity is not adequately described by a single notion of label noise. A reward-aligned expert may be imperfect yet informative; an uninformative expert contributes no consistent direction; and a systematically anti-aligned expert provides a coherent signal with the opposite orientation. These cases call for three distinct responses: \textbf{trust} reward-aligned feedback, \textbf{downweight} uninformative feedback, and \textbf{flip} systematically anti-aligned feedback. The distinction is particularly important in interactive PBRL, where errors in the learned reward affect the policy, the trajectories subsequently queried, and therefore the future training distribution. 

Existing work addresses complementary parts of this problem. Robust PBRL methods operate primarily at the level of individual comparisons~\citep{rime,mcp,r3m,gong2025adaptive}. Crowd-ranking models estimate annotator reliability, including signed coefficients for adversarial workers, but learn scores for a fixed collection of items rather than a functional reward used for policy optimization~\citep{chen2013crowdbt,jin2020heterogeneous,shejole2026signal}. Expertise-aware reward learning has considered neural rewards in offline or non-adversarial settings~\citep{daniels2022expertise,czempin2025expertise}, while crowd-based RL methods either exploit repeated labels for the same query or aggregate judgments from multiple users before reward learning~\citep{zhang2022crowds,chhan2024crowdpref}. The setting studied here combines the aspects left separate by these lines: interactive trajectory-level PBRL with a shared neural reward, persistent signed source-specific trust, explicit treatment of the resulting global reward--source orientation ambiguity in downstream policy optimization, and source-specific labeled pair sets that may be disjoint.

We propose \textbf{TriTrust-PBRL (TTP)}, which extends PEBBLE~\citep{pebble} by learning a signed trust value for each expert alongside a shared reward model. This single parameter lets TTP trust, ignore, or flip each expert's feedback, without requiring repeated labels for the same query. Even with normalized signed trust, however, the reward is identifiable only up to an additive constant and a global sign flip. The constant leaves every reward difference, and hence every preference, unchanged, but the sign flip in PBRL corresponds to learning the wrong behavior~\citep{jin2020heterogeneous,shejole2026signal}; we analyze this ambiguity and show how TTP selects the correct orientation when most experts are reliable.

Empirically, TTP identifies which experts to trust, ignore, or flip, and turns this into better policies. In synthetic experiments it recovers the correct orientation in about \(99\%\) of reliable-majority settings, and on four continuous-control tasks it outperforms PEBBLE, RIME, and MCP.

Our contributions are:
\begin{itemize}
    \item a formulation of interactive multi-expert PBRL that couples a neural trajectory reward with persistent signed source rationalities and supports singly labeled, potentially disjoint expert queries;
    \item an identifiability analysis showing that, under comparison overlap and max-normalized trust, the reward is determined up to an additive constant and a global sign, and the expert trusts up to that same sign; together with a stabilized two-path objective and an empirical characterization of the regime in which it selects the correct branch; and
    \item controlled synthetic diagnostics and end-to-end policy experiments that identify the operating regime of the method and quantify its gains under noisy and adversarial feedback.
\end{itemize}

Together, these results show that source identity can be exploited without repeated labeling of each query, provided that the reward orientation is selected correctly. The remainder of the paper reviews related work, presents TTP and its branch-selection analysis, and reports the synthetic and end-to-end evaluations.

\section{Related Work}

\label{sec:related}

Research most closely related to TTP spans three overlapping areas: (i) fixed-item crowd ranking with annotator-specific reliability~\citep{chen2013crowdbt,jin2020heterogeneous,shejole2026signal}; (ii) source-aware preference and reward learning for decision making and reinforcement learning~\citep{yamagata2021multiple,yamagata2024relative,czempin2025expertise,daniels2022expertise,barnett2023active,freedman2026active,zhang2022crowds,chhan2024crowdpref,yuan2026prefmoe}; and (iii) robustness to unreliable or corrupted preferences~\citep{rime,mcp,r3m}. These lines model different aspects of heterogeneous feedback, including adversarial sources, learned reliability, reward learning, and downstream policy optimization. To the best of our knowledge, no prior method jointly combines persistent signed source-specific rationality coefficients with a shared neural trajectory reward, explicitly analyzes the resulting global reward--source orientation ambiguity in downstream policy optimization, and uses this formulation in interactive trajectory-level PBRL with source-specific labeled pair sets that may be disjoint. TTP combines these elements by jointly learning a neural trajectory reward and signed trust/rationality parameters and using the resulting reward for policy optimization.

\subsection{Fixed-Item Crowd Ranking}

The first line of work ranks a fixed collection of items from crowd comparisons while estimating annotator reliability. Two relevant constructions augment the Bradley--Terry likelihood with source-specific parameters. A correctness-channel model takes the form

\begin{equation}
\label{eq:related-correctness-channel}
P_k^{\mathrm{flip}}(i \succ j)
=q_k\sigma(\Delta s)+(1-q_k)\sigma(-\Delta s),
\end{equation}

while a trust-scaled Bradley--Terry model takes the form

\begin{equation}
\label{eq:related-trust-scaled}
P_k^{\mathrm{signed}}(i \succ j)
=\sigma(\alpha_k\Delta s).
\end{equation}

Here, $i$ and $j$ index two items from the fixed collection, $P_k^{\mathrm{flip}}(i \succ j)$ and $P_k^{\mathrm{signed}}(i \succ j)$ denote the probability that annotator $k$ reports $i$ as preferred to $j$ under each model, $\Delta s = s_i - s_j$ is the difference between their latent scores $s_i$ and $s_j$ shared across annotators, $q_k$ is an annotator's probability of reporting the model-consistent label, $\alpha_k$ is a signed trust/rationality parameter, and $\sigma(z)=\tfrac{1}{1+e^{-z}}$ is the logistic sigmoid. In Eq.~\eqref{eq:related-trust-scaled}, the sign of $\alpha_k$ determines alignment or reversal, while $|\alpha_k|$ controls sensitivity to the score difference.

Related signed-expertise models such as GLAD~\citep{whitehill2009whose} likewise use positive, zero, and negative worker expertise for reliable, uninformative, and systematically adversarial annotators, respectively, but estimate latent labels in a fixed annotation setting rather than a functional trajectory reward coupled to policy optimization. Crowd-BT~\citep{chen2013crowdbt} uses the correctness-channel construction and can identify and correct systematically reversed feedback from malicious annotators. The heterogeneous Thurstone/Bradley--Terry formulation of \citet{jin2020heterogeneous} uses source-specific signed coefficients, explicitly accommodates adversarial annotators, and recognizes the corresponding global sign-reversal ambiguity, while \citet{shejole2026signal} jointly estimate item rewards and worker competencies using an EM procedure and consider unreliable and adversarial workers. These methods establish that annotator-specific reliability and systematic reversal can be inferred jointly with latent item scores. However, they assume a fixed set of items and do not learn a functional trajectory reward that changes the data distribution through downstream policy optimization. In TTP, $\Delta s$ is replaced by the learned trajectory-reward difference $R(\tau_i)-R(\tau_j)$. This makes the symmetry $(R,\alpha)\mapsto(-R,-\alpha)$ operationally important because likelihood-equivalent parameterizations induce policies that optimize opposite reward orientations.

\subsection{Source-Aware Preference Learning and Reinforcement Learning}

A broader line of work considers feedback sources with different skills, rationalities, or areas of expertise. \citet{yamagata2021multiple} study reinforcement learning from action advice supplied by multiple trainers, estimate trainer consistency online, and can exploit both reliable and systematically incorrect trainers. Relatively Rational~\citep{yamagata2024relative} jointly estimates shared utilities and relative trainer rationalities (rationality ratios) from pairwise action comparisons and uses the learned utilities for policy construction. These approaches operate at the action or utility level rather than learning a neural trajectory reward inside an interactive PBRL loop.

\citet{czempin2025expertise} jointly learn a neural reward and annotator expertise from preferences in an offline contextual-bandit setting. Their Bradley--Terry likelihood scales reward differences by an annotator-specific coefficient $\beta_i(s)$, making their constant-expertise formulation closely related to TTP's trust parameter. Their main setting assumes non-adversarial annotators, uses non-negative expertise, and selects checkpoints on a trusted validation subset. Thus, they model learned and context-dependent expertise but do not study persistent negative trust or the associated reward--trust sign ambiguity in downstream policy optimization.

\citet{daniels2022expertise} study context-dependent teacher expertise directly in an interactive PBRL setting based on PEBBLE. Their evaluated expertise-aware methods use specified synthetic expertise functions, while estimating expertise from observed feedback is left as a separate problem. \citet{barnett2023active} assume known teacher rationalities and use value of information to select which teacher should answer each preference query. \citet{freedman2026active} later allow inference over teacher characteristics but remain focused on active teacher selection rather than persistent signed trust. Deep Crowd-BT~\citep{zhang2022crowds} and Crowd-PrefRL~\citep{chhan2024crowdpref} are closest to TTP in downstream structure because both couple crowd preference modeling to policy learning. Deep Crowd-BT learns trajectory rewards for batch RL and models per-comparison reliability using annotator identity, co-labels, and the current reward prediction. Crowd-PrefRL proposes user-reliability estimation and aggregation of judgments supplied by multiple users for each query within online PBRL; its evaluated setting has all users label each query and can also identify minority viewpoints. Their evaluated formulations therefore rely on repeated or co-labeled judgments rather than a persistent signed source coefficient in each individual preference likelihood.

Recent work PrefMoE~\citep{yuan2026prefmoe} also addresses heterogeneous preference supervision by learning multiple latent neural reward experts with trajectory-dependent soft routing for downstream offline PbRL. In contrast, TTP's expert parameters correspond to identifiable feedback sources: TTP learns a shared trajectory reward together with one persistent signed rationality coefficient per source, enabling explicit source-level suppression or inversion and inducing the global reward--source sign ambiguity analyzed here.


\subsection{Robustness to Unreliable Preferences}

A complementary line of work treats unreliable feedback primarily at the level of individual preference observations. PEBBLE~\citep{pebble} provides the interactive reward-and-policy backbone used in our experiments, while B-Pref~\citep{lee2021bpref} evaluates PBRL under several imperfect-teacher models; neither infers persistent source-specific trust.

More recent methods explicitly improve robustness to unreliable preferences. RIME~\citep{rime} and TREND~\citep{huang2025trend} identify or reduce the influence of likely noisy comparisons, while MCP~\citep{mcp} uses preference augmentation to mitigate corrupted labels. Adaptive confidence-aware PBRL~\citep{gong2025adaptive} models reliability at the comparison level, and robust reward-learning methods such as R$^3$M~\citep{r3m} and HSBC~\citep{xie2025hsbc} reduce the influence of erroneous supervision through robust objectives or conservative hypothesis-space cuts. These approaches primarily model reliability at the level of individual comparisons rather than through a persistent source identity.

Reliability-aware methods also appear in language-model alignment. RAPL~\citep{singhal2025rapl} estimates the reliability of individual preference judgments and uses these estimates to weight reward-model learning. Robust DPO~\citep{chowdhury2024rdpo} instead assumes random preference-label flips and modifies the DPO objective to correct for their effect. RE-PO~\citep{cao2026repo} jointly estimates annotator reliability and per-label correctness through an EM-style procedure, using these estimates to reweight preference-optimization updates according to inferred correctness and reliability. RGPO~\citep{huang2026rgpo} estimates annotator reliability and latent consensus preferences from inconsistent human feedback and scales policy updates according to reliability-aware agreement. These methods address unreliable feedback in LLM alignment, whereas TTP learns an explicit trajectory reward together with persistent signed source trust for interactive PBRL, where the coupled reward--trust parameterization introduces the global orientation ambiguity studied here.

\section{Proposed Method: Joint Reward and Expert Trust Learning}
\label{sec:method}

We present \textbf{TriTrust-PBRL (TTP)}, a method that jointly learns a shared reward function and a trust value for each expert from heterogeneous preference data by gradient descent on a single objective. As training proceeds, an expert's trust value moves toward positive if its feedback agrees with the reward the model currently encodes, toward zero if its feedback appears uncorrelated with that reward, and toward negative if its feedback is consistently opposed to it, so that inverting the expert's labels recovers usable signal instead of discarding it.

Figure~\ref{fig:ttp_overview} gives an overview of the method. Throughout, we assume that reliable experts outnumber adversarial ones (the reliable-majority condition). The rest of this section formalizes the model, analyzes how trust values evolve under gradient descent, and gives the practical training recipe.

\subsection{Model and Objective}
\label{sec:setup}

\begin{figure}[t]
    \centering
    \includegraphics[width=1\linewidth]{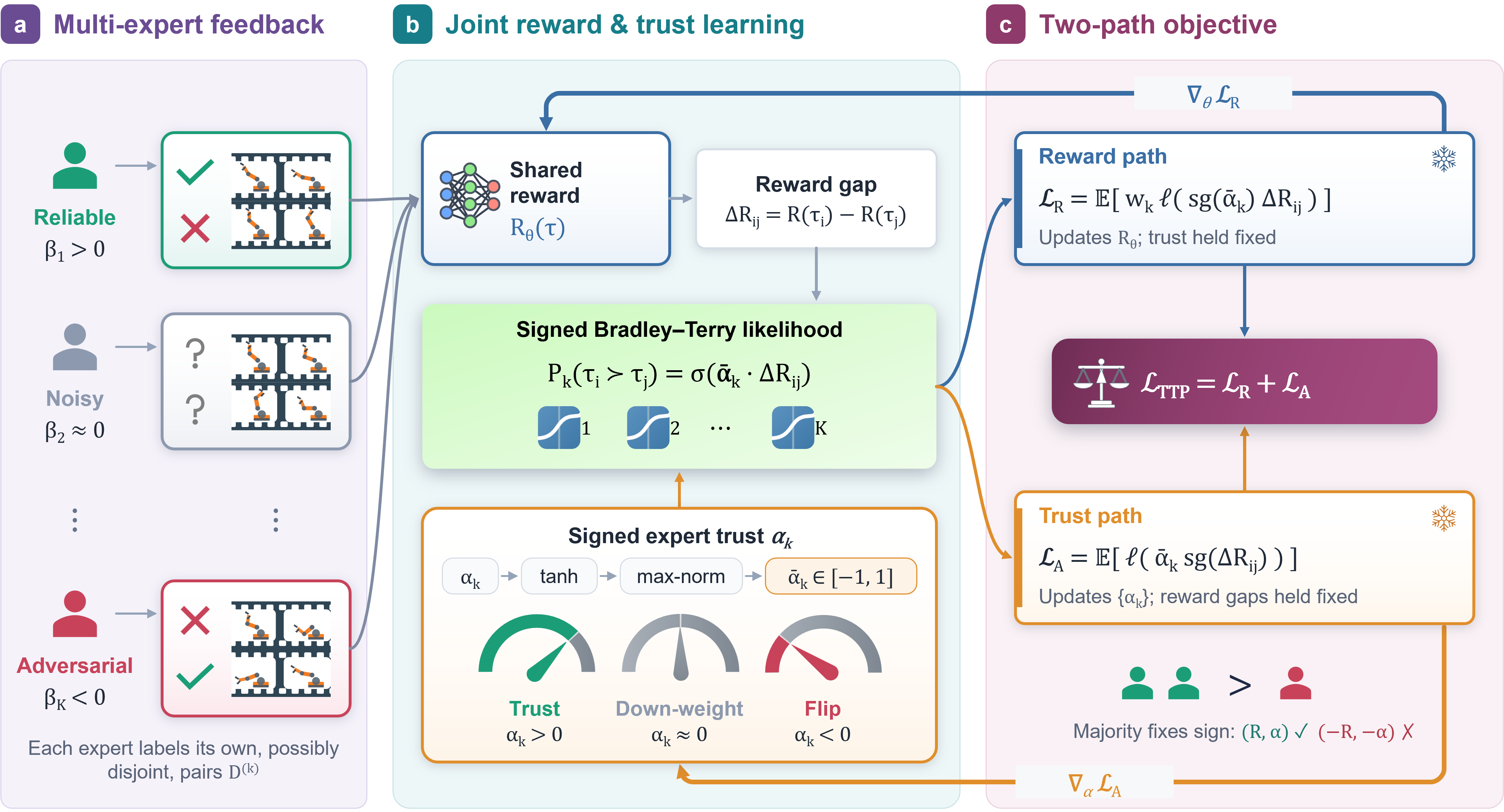}
    \caption{\textbf{Overview of TriTrust-PBRL (TTP).} A shared reward \(R(\tau)\) and expert trusts \(\{\alpha_k\}\) are learned jointly from heterogeneous preferences. Under a correct reward branch, reliable experts develop \(\alpha_k>0\), noisy experts \(\alpha_k\approx 0\), and adversarial experts \(\alpha_k<0\) (inversion). The two-path objective (Eqs.~\eqref{eq:nll}, \eqref{eq:loss-R}--\eqref{eq:loss-A}) aggregates feedback with detached confidence weights on the reward path; Section~\ref{sec:theory} discusses the residual sign ambiguity that makes a reliable majority necessary for the ``correct'' branch.}
    \label{fig:ttp_overview}
\end{figure}

We consider \(K\) experts \(\{e_1,\dots,e_K\}\) who each provide preferences over their own private sets of trajectory pairs. 
Expert \(k\) provides binary preferences over its dataset \(D^{(k)}=\{(\tau_i,\tau_j,y^{(k)}_{ij})\}\), where \(y^{(k)}_{ij} \in \{0,1\}\) indicates whether expert \(k\) prefers trajectory \(\tau_i\) over \(\tau_j\):
\[
y^{(k)}_{ij} = \begin{cases}
1 & \text{if expert } k \text{ prefers } \tau_i \text{ over } \tau_j, \\
0 & \text{otherwise.}
\end{cases}
\]
The goal is to jointly learn a \emph{global} reward function \(R(\tau)\) that maps trajectories to scalar rewards, and expert-specific trust parameters \(\{\alpha_k\}_{k=1}^K\) that capture the reliability (and potential adversariality) of each expert's feedback.
Here, $R(\tau)=\sum_{t} \hat r_\theta(s_t,a_t)$ denotes the cumulative trajectory reward obtained by summing per-step rewards produced by a learned reward model $\hat r_\theta$. Where the distinction matters, we write $\hat R$ for this learned reward function under active optimization, to contrast it with the ground-truth intended reward $R^*$.

We model the probability that expert \(k\) prefers \(\tau_i\) over \(\tau_j\) as:
\begin{equation}
\label{eq:expert-logistic}
P\!\left(y^{(k)}_{ij}=1 \mid R(\tau_i),R(\tau_j),\alpha_k\right)
=
\sigma\!\big(\alpha_k \, (R(\tau_i)-R(\tau_j))\big).
\end{equation}

We learn \(R(\cdot)\) and \(\{\alpha_k\}\) by minimizing the negative log-likelihood:
\begin{equation}
\label{eq:nll}
\mathcal{L}(R,\{\alpha_k\}) = -\sum_{k=1}^{K} \sum_{(\tau_i,\tau_j)\in D^{(k)}} 
\Big[
y^{(k)}_{ij}\log \sigma(\alpha_k \Delta R_{ij})
+ (1-y^{(k)}_{ij}) \log\big(1-\sigma(\alpha_k \Delta R_{ij})\big)
\Big],
\end{equation}
where \(\Delta R_{ij} = R(\tau_i)-R(\tau_j)\).

This objective reduces to the standard Bradley--Terry loss~\citep{bradley1952rank} when \(\alpha_k = 1\) for all experts, but with learnable \(\alpha_k\) it adaptively reweights each expert according to its inferred trustworthiness. Because \(R(\cdot)\) and \(\{\alpha_k\}\) are optimized jointly, the reward model can draw on reliable experts while remaining robust to unreliable or adversarial ones: reliable experts develop larger \(\alpha_k\) values, increasing their influence on the learned reward; noisy experts see their \(\alpha_k\) shrink toward zero, down-weighting their inconsistent preferences; and adversarial experts who systematically provide incorrect feedback can develop negative \(\alpha_k\) values, causing the model to flip their preferences into alignment with the global reward. 

We instantiate TTP on top of PEBBLE~\citep{pebble}, but since it only modifies the reward-learning loss, it can be combined with other preference-based RL methods that learn a reward model from pairwise preferences. Instead of treating every label as coming from a single equally reliable teacher, this loss scales each expert's reward difference by a learned signed trust $\alpha_k$; policy learning is left unchanged.


\subsection{Gradient Analysis of Expert Trust Dynamics}
\label{sec:alpha-analysis}

To understand how trust parameters evolve during training, we analyze the gradient of the loss with respect to \(\alpha_k\). Differentiating Eq.~\eqref{eq:nll} w.r.t.\ \(\alpha_k\) yields:
\begin{equation}
\label{eq:alpha-grad}
\frac{\partial \mathcal{L}}{\partial \alpha_k}
=
\sum_{(\tau_i,\tau_j)\in D^{(k)}}
\big[\sigma(\alpha_k \Delta R_{ij}) - y^{(k)}_{ij}\big] \cdot \Delta R_{ij}.
\end{equation}

Each term weights the reward gap \(\Delta R_{ij}\) by the prediction error \(\sigma(\alpha_k\Delta R_{ij})-y^{(k)}_{ij}\). Holding \(R\) fixed and taking expectations over expert \(k\)'s labels, the direction of the expected update depends only on how those labels correlate with \(\mathrm{sign}(\Delta R_{ij})\) (Appendix~\ref{app:noisy-expert-gradient} gives the derivation):
\begin{itemize}
  \item \textbf{Reliable experts}, whose labels agree with \(\mathrm{sign}(\Delta R_{ij})\), push \(\alpha_k\) upward, increasing their influence on the reward.
  \item \textbf{Noisy experts}, whose labels are uncorrelated with \(\Delta R_{ij}\), have zero expected gradient at \(\alpha_k=0\), and at any \(\alpha_k\neq0\) an expected gradient with the sign of \(\alpha_k\), so the expected update points toward zero.
  \item \textbf{Adversarial experts}, whose labels oppose \(\mathrm{sign}(\Delta R_{ij})\), push \(\alpha_k\) negative; a negative \(\alpha_k\) flips their preferences into signal aligned with the reward.
\end{itemize}

\paragraph{Interpretation and scope.}
This analysis describes only the \emph{average} direction of the trust update while the reward is held fixed. For a noisy expert, it says that trust has no average push away from zero, and that once trust is nonzero the average push is back toward zero. It does not guarantee that trust actually converges to zero in training, for two reasons. First, each update uses a small random batch of labels, so individual steps are noisy and can overshoot. Second, the reward is learned at the same time as trust, so \(\Delta R_{ij}\) keeps changing. In our experiments, trust nonetheless separates as this analysis suggests (Section~\ref{sec:synthetic-branch}); Section~\ref{sec:cross-entropy-weight} describes how we stabilize training in practice.

Dependence on \(R\) is unavoidable: every term in Eq.~\eqref{eq:alpha-grad} involves \(\Delta R_{ij}\), so an expert's apparent quality is judged relative to whichever ranking \(R\) currently encodes. Were the reward instead converging toward \(-R^*\), the same labels would reverse every trust gradient and swap reliable and adversarial roles; these trust-gradient directions are therefore informative only once \(\mathrm{sign}(\Delta R_{ij})\) reflects the true ranking. This raises a natural question: since \(R\)'s own updates are shaped by the very trust values whose sign we have just shown to depend on \(R\), how does \(R\) come to encode the correct ranking in the first place?

The answer is that \(\theta\) and \(\{\alpha_k\}\) are not resolved one after the other but by the same gradient step at every iteration; there is no stage in which one is optimized while the other is held fixed. What changes over the course of training is the relative size of the two signals, not their order. Early on, \(\Delta R_{ij}\) is small by construction, since Section~\ref{sec:cross-entropy-weight} initializes the reward at \(R\equiv 0\), so the trust gradient in Eq.~\eqref{eq:alpha-grad} is small for every expert regardless of \(\alpha_k\), while \(\theta\) still receives a gradient from every expert with equal weight: because that aggregate gradient combines equally weighted contributions, its net direction is set by whichever side, reward-aligned or adversarial, is more numerous. As \(\hat R\) grows and its overall direction stabilizes, \(\Delta R_{ij}\) becomes informative and the trust gradient strengthens, so \(\alpha_k\) separates by expert type. These are two regimes of one coupled dynamical system evolving together, not two sequential phases of an algorithm; Section~\ref{sec:synthetic-branch} studies this coupled dynamics directly, offline and without a policy in the loop.

\subsection{Practical Enhancements}
\label{sec:cross-entropy-weight}

The gradient analysis above holds the reward fixed to isolate how a single trust value responds to it; in the actual algorithm, \(\theta\) and \(\{\alpha_k\}\) are updated together at every step. Directly optimizing Eq.~\eqref{eq:nll} this way runs into three practical problems. We take them up one at a time below, pairing each problem with the fix it motivates.

\begin{enumerate}
  \item \textbf{Unanchored trust scale.} Each pair's reward gradient, \(\alpha_k[\sigma(\alpha_k\Delta R_{ij})-y^{(k)}_{ij}]\), is proportional to trust, but the likelihood does not fix the scale of trust. At initialization (\(\alpha_k^{(0)}=0.01\)) this leaves the reward with almost no gradient. Later, the trust of a consistent expert can grow without bound, letting a few of its mispredicted pairs dominate the reward update. We therefore use
  \[
  \tilde{\alpha}_k = \tanh(\alpha_k),
  \qquad
  \bar{\alpha}_k = \frac{\tilde{\alpha}_k}{\max_{k'} |\tilde{\alpha}_{k'}|}
  \]
  (denominator detached) in place of \(\alpha_k\) inside the sigmoid. Max-normalization fixes the scale: the most-trusted expert has \(|\bar\alpha_k|=1\), so all trusts lie in \([-1,1]\), and at initialization \(\bar\alpha_k=1\) for every expert, giving an ordinary Bradley--Terry step on \(R\). It also removes the scale ambiguity between \(R\) and \(\{\alpha_k\}\), leaving a global sign and an additive offset in \(R\) that no reward difference depends on (Corollary~\ref{cor:maxnorm-ident}). The \(\tanh\) map makes strongly trusted experts saturate near \(\pm1\), so reliable experts receive similar weight rather than weight that depends on how far their raw parameters have drifted. 
  \item \textbf{Weak suppression of low-trust experts.} Trust already scales each pair's reward gradient by \(\bar\alpha_k\):
  \[
  \nabla_\theta\, \ell_{ij}^{(k)}\!\big(\bar\alpha_k\Delta R_{ij}\big)
  =
  \bar\alpha_k\big[\sigma(\bar\alpha_k\Delta R_{ij})-y^{(k)}_{ij}\big]\,\nabla_\theta \Delta R_{ij},
  \]
  but only linearly, and max-normalization starts every expert at \(\bar\alpha_k=1\). While trust is still separating, a noisy expert whose \(\bar\alpha_k\) has only partly shrunk therefore still pushes the reward with substantial weight, even though Eq.~\eqref{eq:alpha-grad} shows its labels carry no useful directional signal. We add a confidence weight on the reward loss,
  \[
  w_k(t) = \frac{K\,|\tilde{\alpha}_k|}{\sum_{k'=1}^{K} |\tilde{\alpha}_{k'}|},
  \]
  which equals one for every expert when all trusts are equal and makes each pair's reward gradient scale with \(|\bar\alpha_k|\,\bar\alpha_k\) rather than \(\bar\alpha_k\) (up to a factor shared by all experts). Experts whose trust magnitude shrinks are thus suppressed in \(\mathcal{L}_R\) faster, while large-magnitude adversarial experts stay active because inversion is handled by the sign of \(\bar\alpha_k\) inside the logistic term.
  \item \textbf{Unwanted coupling between the reward and trust updates.} Simply combining the reward and trust losses into one objective creates a subtler problem: the confidence weight \(w_k\) introduced above is itself a function of \(\alpha_k\), so backpropagating through it opens an unintended gradient path from the reward loss back into trust, on top of the one we actually want. We fix this by splitting the objective into two paths with stop-gradients. Let \(\ell_{ij}^{(k)}(\cdot)\) denote the per-pair negative log-likelihood from Eq.~\eqref{eq:nll} as a function of its logit argument. We optimize
  \begin{align}
  \label{eq:loss-R}
  \mathcal{L}_R
  &=
  \mathbb{E}_{k,\,(\tau_i,\tau_j)\sim D^{(k)}}
  \Big[
  w_k(t)\,
  \ell_{ij}^{(k)}\!\big(\mathrm{sg}(\bar{\alpha}_k)\,\Delta R_{ij}\big)
  \Big],
  \\
  \label{eq:loss-A}
  \mathcal{L}_A
  &=
  \mathbb{E}_{k,\,(\tau_i,\tau_j)\sim D^{(k)}}
  \Big[
  \ell_{ij}^{(k)}\!\big(\bar{\alpha}_k\,\mathrm{sg}(\Delta R_{ij})\big)
  \Big],
  \end{align}
  where \(\mathrm{sg}(\cdot)\) stops gradients, and set
  \begin{equation}
  \label{eq:weighted-loss-basic}
  \mathcal{L}_{\mathrm{TTP}} = \mathcal{L}_R + \mathcal{L}_A.
  \end{equation}
  Thus \(\mathcal{L}_R\) updates \(R\) through weighted, trust-scaled preferences with frozen \(\bar\alpha_k\); \(\mathcal{L}_A\) updates \(\{\alpha_k\}\) through unweighted cross-entropy with frozen reward gaps. 
\end{enumerate}

\noindent
Combining these components yields the practical procedure in Algorithm~\ref{alg:joint-learning}.

\begin{algorithm}[h]
\caption{TriTrust-PBRL (TTP) algorithm}
\label{alg:joint-learning}
\begin{algorithmic}[1]
\Require Expert datasets $\{D^{(k)}\}_{k=1}^K$, SGD rates $\eta_R, \eta_\alpha$, trust step cap $\delta_{\bar\alpha}$ (optional)
\Ensure Learned reward $R(\cdot)$ and trust parameters $\{\alpha_k\}$
\State Initialize reward parameters of $R$ at random; zero the final linear layer of each reward member; set $\alpha_k \leftarrow 0.01$ for all $k$
\For{iteration $t = 1, 2, \ldots, T$}
    \State Sample a minibatch of preference triples $(\tau_i,\tau_j,y^{(k)}_{ij})$ from $\{D^{(k)}\}_{k=1}^K$
    \State Compute reward differences $\Delta R_{ij} \leftarrow R(\tau_i) - R(\tau_j)$ for all sampled pairs
    \State $\tilde{\alpha}_k \leftarrow \tanh(\alpha_k)$;\quad $\bar{\alpha}_k \leftarrow \tilde{\alpha}_k / \max_{k'} |\tilde{\alpha}_{k'}|$ with detached denominator
    \State $w_k \leftarrow \frac{K\ |\tilde{\alpha}_k|}{\sum_{k'} |\tilde{\alpha}_{k'}|}$;\quad detach $w_k$ for the reward path
    \State $\mathcal{L}_R \leftarrow$ mean over the minibatch of $w_k\,\ell_{ij}^{(k)}(\mathrm{sg}(\bar{\alpha}_k)\,\Delta R_{ij})$ (Eq.~\eqref{eq:loss-R})
    \State $\mathcal{L}_A \leftarrow$ mean over the minibatch of $\ell_{ij}^{(k)}(\bar{\alpha}_k\,\mathrm{sg}(\Delta R_{ij}))$ (Eq.~\eqref{eq:loss-A})
    \State Backpropagate $\mathcal{L}_{\mathrm{TTP}}=\mathcal{L}_R+\mathcal{L}_A$
    \State $R \leftarrow R - \eta_R \nabla_R \mathcal{L}_{\mathrm{TTP}}$;\quad $\alpha_k \leftarrow \alpha_k - \eta_\alpha \nabla_{\alpha_k} \mathcal{L}_{\mathrm{TTP}}$
    \State If $\delta_{\bar\alpha}>0$, project each $\bar{\alpha}_k$ so $|\bar{\alpha}_k(t)-\bar{\alpha}_k(t{-}1)|\le \delta_{\bar\alpha}$
\EndFor
\State \textbf{return} $R(\cdot)$ and $\{\alpha_k\}$
\end{algorithmic}
\end{algorithm}

\section{Symmetry, Branch Selection, and Identifiability Limits}
\label{sec:theory}
The gradient analysis of Section~\ref{sec:alpha-analysis} shows how trust separates by expert type once the learned reward encodes the true ranking, but it takes that ranking as given. This leaves two questions unanswered. First, what, if anything, guarantees that the reward converges to the true ranking rather than its reverse? Second, before addressing estimation, are the reward and trust parameters even identifiable from observed preference data? This section answers both. We first characterize the sign symmetry that survives bounded, normalized trust and show that resolving it is a matter of optimization dynamics, captured by a reliable-majority condition, rather than of identifiability. We then turn to identifiability itself: Theorem~\ref{thm:identifiability} shows that a fixed comparison graph determines the reward and trust parameters up to an affine reparameterization (proof in Appendix~\ref{app:identifiability-proof}), and a closing remark examines what a shared reward parameterization buys when, as in our policy-learning experiments, no two experts ever label the same pair.

\paragraph{Trust as a signed rationality parameter.}
In the per-expert logistic model (Eq.~\eqref{eq:expert-logistic}), each observation depends on the product \(\alpha_k\,\Delta R_{ij}\), where \(\Delta R_{ij}=R(\tau_i)-R(\tau_j)\). This makes \(\alpha_k\) a signed \emph{trust/rationality} parameter: \(\alpha_k>0\) aligns expert \(k\) with the learned reward differences, \(|\alpha_k|\approx 0\) down-weights an effectively random expert, and \(\alpha_k<0\) flips a systematic anti-alignment.

\paragraph{Sign ambiguity.}
Because the likelihood depends only on products \(\alpha_k\,\Delta R_{ij}\), the transformation \((R,\{\alpha_k\})\mapsto(-R,\{-\alpha_k\})\) leaves all comparison probabilities unchanged. With max-normalized bounded trust (\(\tilde\alpha_k=\tanh(\alpha_k)\), \(\bar\alpha_k=\tilde\alpha_k/\max_{k'}|\tilde\alpha_{k'}|\)), this global sign ambiguity remains and corresponds to two equally likely ``branches'' for the reward.

\paragraph{Reliable-majority condition (practical branch selection).}
Branch selection is resolved by optimization dynamics rather than by identifiability alone, and what resolves it is how the reward is initialized. Raw trust starts equal across experts, so max-normalization (Section~\ref{sec:cross-entropy-weight}) gives \(\bar\alpha_k^{(0)}=1\) for every expert; the two initializations we compare differ only in the reward. \emph{Stabilized} sets \(\theta{=}0\), so \(\Delta R\equiv 0\) at the first preference update, whereas \emph{Standard} uses a large random init, so \(\Delta R\) already has substantial magnitude before any label is seen. Under Stabilized, every expert therefore enters the first reward step with the same trust and the same zero reward difference, which makes that step an approximately uniform Bradley--Terry fit: a reliable majority pulls \(\hat R\) toward \(R^*\) rather than \(-R^*\). Once \(\hat R\) carries even coarse ranking signal, Eq.~\eqref{eq:alpha-grad} reinforces that majority by raising trust for experts that agree with it and suppressing or inverting those that do not. This is a majority rule rather than a correctness guarantee: when systematic adversaries outnumber reliable experts, the same dynamics select the flipped branch. Standard removes the premise of the argument, because its large random \(\Delta R\) points in an arbitrary direction before any expert has voted, so the first trust update follows whichever experts happen to agree with that direction rather than whichever ones are numerous, and the same feedback loop then locks in the branch the initialization favored. Empirically, Stabilized follows the majority rule while Standard often does not (Section~\ref{sec:synthetic-branch}).

Resolving the sign ambiguity is a separate question from how identifiable \((R,\{\alpha_k\})\) are in the first place. We turn to that question next.

\paragraph{Identifiability (offline).}
The next theorem formalizes what a fixed comparison graph can determine about \((R,\{\alpha_k\})\) up to the unavoidable global affine ambiguity; a full proof and the max-normalization corollary are given in Appendix~\ref{app:identifiability-proof}.

\begin{theorem}[Identifiability]
\label{thm:identifiability}
Consider the logistic preference model in Eq.~\eqref{eq:expert-logistic}, in which each expert $k$ provides labels according to
\[
P\!\left(y^{(k)}_{ij}=1 \mid R,\alpha_k\right) = \sigma\!\big(\alpha_k (R(\tau_i)-R(\tau_j))\big).
\]
Assume the following conditions hold:
\begin{enumerate}
    \item \textbf{Global trajectory connectedness:} Let $G = (V,E)$ be the ``trajectory graph'' whose vertices $V$ are trajectories and whose undirected edges $E$ are pairs $\{\tau_i,\tau_j\}$ that appear in at least one expert's dataset $D^{(k)}$ with $\alpha_k \neq 0$. Then $G$ is connected.
    \item \textbf{Expert connectedness:} Let $\{e_1, \ldots, e_K\}$ denote the set of experts, and let $\mathcal{E} = \{ e_k : \alpha_k \neq 0 \}$ be the subset of experts with nonzero trust parameters. Define an undirected ``expert graph'' on $\mathcal{E}$ in which two experts $e_k$ and $e_l$ are connected if and only if there exists a comparison $(\tau_i,\tau_j)$ that appears in both datasets $D^{(k)}$ and $D^{(l)}$ and satisfies $R(\tau_i) \neq R(\tau_j)$. We assume that this graph is connected.
\end{enumerate}
Then the model is identifiable up to a global affine reparameterization: for any two parameter sets $(R,\{\alpha_k\})$ and $(R',\{\alpha_k'\})$ that yield the same likelihood for all observed preferences, there exist constants $a \neq 0$ and $b \in \mathbb{R}$ such that
\[
R'(\tau) = a\,R(\tau) + b \quad \text{for all } \tau, 
\qquad
\alpha_k' = \alpha_k / a \quad \text{for all } k.
\]
Equivalently, $R(\cdot)$ is identifiable up to an affine transformation and the trust parameters are identifiable up to the corresponding common (possibly sign-changing) scaling.
\end{theorem}

Under the bounded, max-normalized trust of Section~\ref{sec:cross-entropy-weight}, this affine ambiguity collapses further: Corollary~\ref{cor:maxnorm-ident} shows that the continuous scale is fixed, so $(R,\bar\alpha)$ is identifiable up to an additive constant in $R$ and a global sign $\pm 1$, i.e.\ $R'=\pm R+b$ and $\bar\alpha'=\pm\bar\alpha$. The constant is harmless because the model depends on $R$ only through differences $\Delta R_{ij}$, which are therefore identified up to the sign alone; the sign is exactly the branch ambiguity discussed above. Appendix~\ref{app:identifiability-proof} contains the full proofs of the theorem and the corollary.

\begin{remark}[Coupling without overlapping comparisons]
\label{rem:theory-disjoint}
Theorem~\ref{thm:identifiability} assumes overlapping comparisons, whereas our policy benchmarks use disjoint labeling, in which each expert annotates a private set of pairs. A shared reward parameterization still couples these datasets: every expert's loss is computed from the same \(\theta\), so \(\nabla_\theta\mathcal{L}=\sum_{k}\nabla_\theta\mathcal{L}^{(k)}(\theta,\alpha_k)\), and fitting one expert's comparisons also changes \(\hat R\) on trajectories that only other experts labeled. This is an optimization mechanism, not an identifiability result.
\end{remark}

\section{Experiments}
\label{sec:experiments}

In this section, we present numerical experiments to evaluate TTP under heterogeneous expert feedback. The evaluation has two stages. Section~\ref{sec:synthetic-branch} uses offline preferences with known \(R^*\) (no MDP) to isolate branch selection, trust separation, and partial adversaries. Section~\ref{sec:policy-experiments} integrates TTP into PEBBLE~\citep{pebble} on four continuous-control tasks with B-Pref teachers~\citep{lee2021bpref}, and compares it to RIME~\citep{rime}, MCP~\citep{mcp}, and oracle SAC~\citep{haarnoja2018soft}. Throughout, the policy experiments remain within the reliable-majority regime that the synthetic diagnostics identify as necessary for correct-branch recovery.

\subsection{Synthetic diagnostics}
\label{sec:synthetic-branch}

In this section, we evaluate the mechanism in isolation, where the ground truth is known: no environment and no policy, only a shared reward head and trusts \(\{\alpha_k\}\). All diagnostics use \(R(\tau)=\sum_{t=1}^{T}\tanh(\theta^\top s_t)\) (\(T{=}50\) unless mentioned; \(d{=}16\)). Preferences are \emph{disjoint} across experts, matching the robotics regime. We study which branch is selected, how the initialization scale and the trajectory length \(T\) affect that outcome, whether trusts separate by expert type, and how the mechanism behaves under partial adversaries.

\subsubsection{Setup: synthetic preference generation}
\paragraph{Data generation.}
We generate the preferences synthetically, so that both the ground-truth ranking and each expert's type are known. We sample random segment features \(s_{i,t}\in\mathbb{R}^d\) for \(n\) abstract trajectories of length \(T{=}50\) and a latent direction \(\theta^*\), defining
\[
R^*(\tau_i)=\sum_{t=1}^{T}\tanh((\theta^*)^\top s_{i,t}),
\]
then standardize \(R^*\) across trajectories. The learner fits the same functional form with parameters \(\theta\). Trajectories are partitioned into \(K\) blocks (one per expert). Each expert labels \(m\) pairs drawn within its block (disjoint feedback). Expert \(k\) answers with rationality \(\beta_k\):
\[
P\!\left(y^{(k)}_{ij}=1\right)
=
\sigma\!\big(\beta_k\,\Delta R^*_{ij}\big),
\qquad
y^{(k)}_{ij}\sim\mathrm{Bernoulli}\!\big(P(y^{(k)}_{ij}=1)\big).
\]
As before, \(\beta_k>0\) is reliable, \(\beta_k=0\) noisy, and \(\beta_k<0\) adversarial (canonical \(\beta_k\in\{+1,0,-1\}\)). We denote a crowd composition by the count of each expert type, writing R for reliable, N for noisy, and A for adversarial: for example, 3R1A denotes three reliable experts and one adversarial expert, and 1R3A denotes one reliable and three adversarial. By default \(n{=}500\), \(m{=}256\), \(d{=}16\). The learner never sees \(R^*\) or \(\{\beta_k\}\).

\paragraph{Parameter setting.}
We minimize the two-path TTP loss (Eqs.~\eqref{eq:loss-R}--\eqref{eq:weighted-loss-basic}) over \(\theta\) and \(\{\alpha_k\}\) with \(\alpha_k^{(0)}=0.01\), \(\eta_\theta{=}0.05\), \(\eta_\alpha{=}0.005\), and \(400\) steps unless noted. We report the initialization scale as \(\mathrm{rms}|\Delta R|_0\), the root-mean-square preference gap \(|\Delta R_{ij}|=|R(\tau_i)-R(\tau_j)|\) over the labeled pairs at step \(0\), which we set to a target value by rescaling \(\|\theta\|\). We compare:
\begin{itemize}
    \item \emph{Standard}: large random init (\(\mathrm{rms}|\Delta R|_0\approx 1.4\));
    \item \emph{Stabilized}: \(\theta{=}0\) so \(\Delta R\equiv 0\) at start.
\end{itemize}

\subsubsection{Branch-symmetry diagnostic}
The first failure mode is a global sign flip. We evaluate when training converges to the correct branch (\(\mathrm{corr}(\hat R,R^*)>0.5\)) rather than the flipped one under the shared head.

For each of \(200\) seeds we train on 3R1N, 3R1A, and 1R3A. Table~\ref{tab:synthetic-branch} reports correct-branch rates; Figure~\ref{fig:synthetic-branch} shows learned trusts on the modal branch for each mixture~(a) to~(c). We observe that:
\begin{itemize}
    \item Stabilized is essentially perfect under a reliable majority (3R1N / 3R1A: \(100\%\)), and locks onto the flipped branch under 1R3A (\(0\%\) correct).
    \item Standard (large init, \(\mathrm{rms}|\Delta R|_0\approx 1.4\)) is far weaker on the same mixtures (3R1N \(18\%\), 3R1A \(36\%\)), and only recovers a correct majority under 1R3A (\(62.5\%\)).
    \item Under a reliable majority, Stabilized recovers high reliable trust with near-zero noisy and negative adversarial trust (3R1N: \(\bar\alpha_R\approx 0.88\), \(\bar\alpha_N\approx 0\); 3R1A: \(\bar\alpha_R\approx 0.84\), \(\bar\alpha_A\approx -0.84\)). Standard more often selects the flipped modal branch, and therefore shows swapped signs on 3R1N and 3R1A.
\end{itemize}
The init scale therefore separates the two variants, as the branch-selection argument of Section~\ref{sec:theory} predicts.

\begin{table}[t]
\centering
\caption{Correct-branch rate (\(\mathrm{corr}(\hat R,R^*)>0.5\)) under the shared head over \(200\) seeds, for \(n=500\).}
\label{tab:synthetic-branch}
\begin{tabular}{lcc}
\toprule
Expert mixture & Standard & Stabilized \\
\midrule
\(3\) reliable / \(1\) noisy (3R1N) & \(18\%\) & \(\mathbf{100\%}\) \\
\(3\) reliable / \(1\) adversarial (3R1A) & \(36\%\) & \(\mathbf{100\%}\) \\
\(1\) reliable / \(3\) adversarial (1R3A) & \(62.5\%\) & \(0.0\%\) \\
\bottomrule
\end{tabular}
\end{table}

\begin{figure}[t]
    \centering
    \begin{subfigure}[b]{0.32\linewidth}
        \centering
        \includegraphics[width=\linewidth]{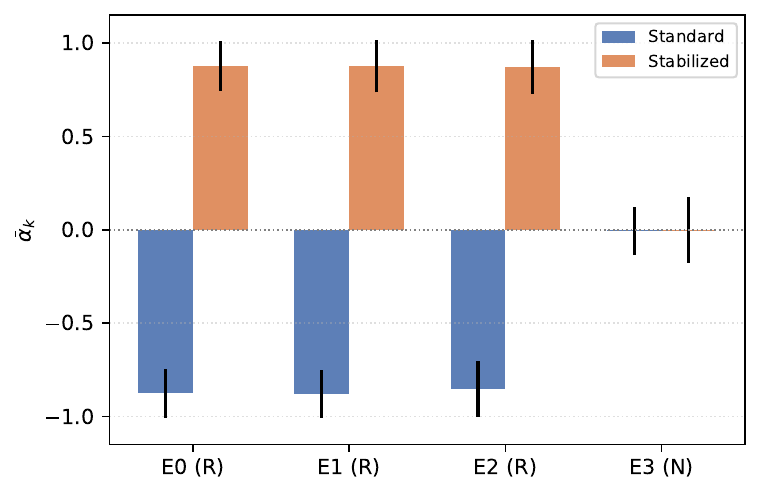}
        \caption{3R1N}
        \label{fig:synthetic-branch-3r1n}
    \end{subfigure}\hfill
    \begin{subfigure}[b]{0.32\linewidth}
        \centering
        \includegraphics[width=\linewidth]{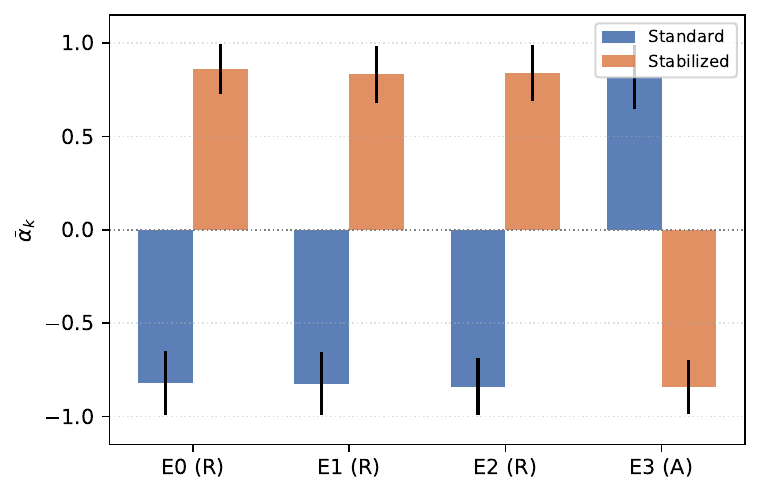}
        \caption{3R1A}
        \label{fig:synthetic-branch-3r1a}
    \end{subfigure}\hfill
    \begin{subfigure}[b]{0.32\linewidth}
        \centering
        \includegraphics[width=\linewidth]{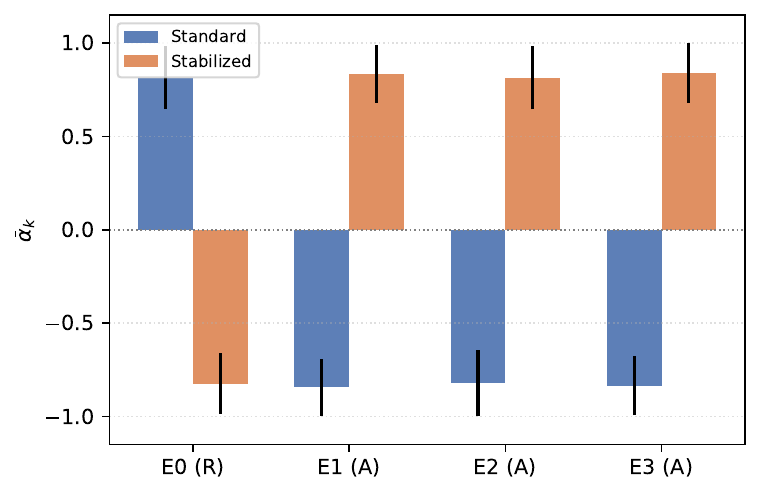}
        \caption{1R3A}
        \label{fig:synthetic-branch-1r3a}
    \end{subfigure}
    \caption{Per-expert normalized trust \(\bar\alpha_k\) under the shared head (mean \(\pm\) std over seeds on the modal branch) over \(200\) seeds, for \(n=500\). Experts are labeled reliable (R), noisy (N), or adversarial (A).}
    \label{fig:synthetic-branch}
\end{figure}

\subsubsection{Sensitivity to initialization scale and trajectory length}
\label{sec:synthetic-sensitivity}
The Standard/Stabilized gap is governed by the initialization scale itself, not by the particular Standard value. Sweeping \(\mathrm{rms}|\Delta R|_0\) from \(0\) to \(6\), every reliable-majority mixture keeps a correct-branch rate of at least \(92\%\) up to \(\mathrm{rms}|\Delta R|_0=0.5\), roughly the gap of a default MLP init, and collapses at the Standard scale and beyond. Initialization scale is therefore a first-order failure mode alongside the reliable-versus-adversarial count. The ranking is also not an artifact of the horizon: it holds for \(T\in\{10,\ldots,200\}\) once \(\eta_\theta\) is scaled with \(T\). Appendix~\ref{app:synthetic-sensitivity} reports both sweeps.

\subsubsection{Expert-ratio sweep (\(K=10\))}
The \(K=4\) mixtures above were chosen in advance. To cover the full space, we repeat the same shared-head experiment with \(K=10\) experts over every possible composition: \(n_A\) adversarial and \(n_N\) noisy experts (\(n_A+n_N\le 10\)), with the rest reliable. This yields \(66\) mixtures; all other settings match the \(K=4\) Stabilized/Standard diagnostic (\(n{=}500\)). Figure~\ref{fig:synthetic-ratio-sweep} plots correct-branch rate~\textbf{(a)} and mean learned trust for adversarial~\textbf{(b)}, noisy~\textbf{(c)}, and reliable~\textbf{(d)} experts, under Standard~\textbf{(i)} and Stabilized~\textbf{(ii)}.

We observe that:
\begin{itemize}
    \item Stabilized~(ii) recovers the correct branch whenever reliable experts strictly outnumber adversarial ones (mean correct-branch \(\approx 99\%\) on that side), and fails once adversaries are in the majority~\textbf{(a)}. Standard~(i) is substantially weaker even under a reliable majority (mean \(\approx 57\%\); no majority cell reaches \(90\%\)).
    \item Stabilized shows a sharp threshold at \(n_R=n_A\): near-perfect recovery on the reliable-majority side, near chance at the tie, and a flipped branch once adversaries are in the majority. This boundary is governed primarily by the reliable-versus-adversarial count; noisy experts (\(\beta_k=0\)) add little directional signal~\textbf{(c)}, though a large noisy fraction can blur cells near the tie.
    \item Mean trusts (averaged over all seeds) indicate which branch was selected~\textbf{(b)},~\textbf{(d)}: under Stabilized in the reliable-majority region, \(\bar\alpha_A\) is strongly negative, \(\bar\alpha_R\) strongly positive, and \(\bar\alpha_N\) near zero. Once adversaries dominate, reliable and adversarial trusts swap sign.
\end{itemize}
White cells in~(b) to~(d) mark mixtures with zero experts of that type, for which the corresponding mean is undefined.

In summary, Stabilized follows the reliable-majority principle on the shared head, with a sharp boundary in the reliable-versus-adversarial count; Standard (large init) does not.

\afterpage{%
\begin{figure}[p]
    \centering
    \includegraphics[width=\linewidth,height=0.88\textheight,keepaspectratio]{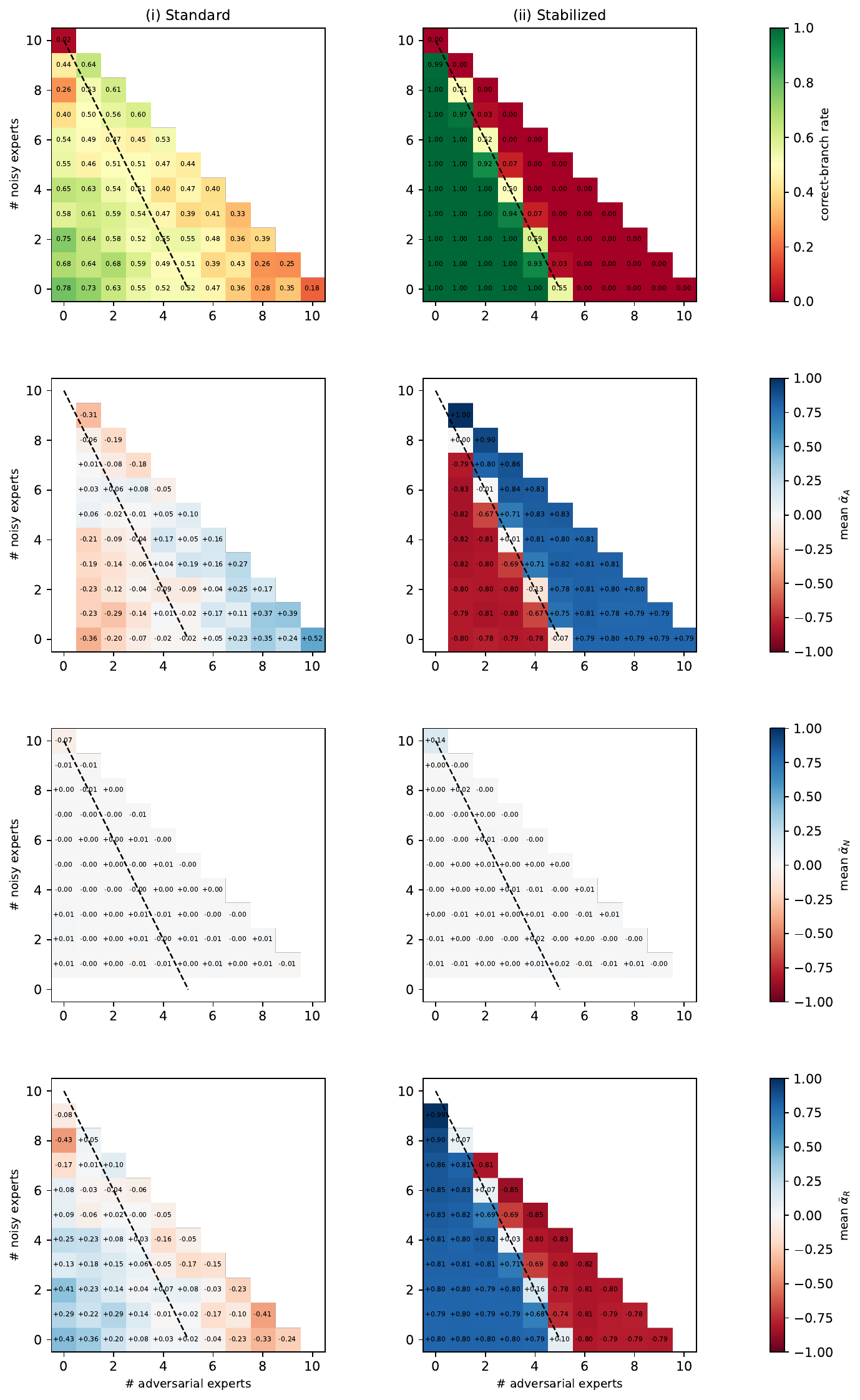}
    \caption{Expert-ratio sweep over every mixture of \(n_A\) adversarial and \(n_N\) noisy experts (remainder reliable) under the shared head with \(K=10\), over \(100\) seeds per cell, for \(n=500\); cell values are annotated. Rows: \textbf{(a)}~correct-branch rate; (b) to~(d)~mean learned trust \(\bar\alpha_A\), \(\bar\alpha_N\), and \(\bar\alpha_R\) (all seeds, not conditioned on branch; white when that type is absent). Columns: \textbf{(i)}~Standard, \textbf{(ii)}~Stabilized. The dashed line marks \(\#\text{reliable}=\#\text{adversarial}\).}
    \label{fig:synthetic-ratio-sweep}
\end{figure}%
}

\subsubsection{Partial and stochastic adversaries}
The robotics experiments use an adversary with (\(\beta_k=-1\)). Here we ask whether the same picture holds for milder corruption under the shared head. On 3R1A we weaken the adversary to \(\beta_4\in\{-0.5,-0.25\}\), or make it stochastic: with probability \(p\in\{0.25,0.5\}\) it behaves as an anti-oracle, otherwise as a reliable expert (\(200\) seeds; \(n{=}500\)). We compare Stabilized and Standard.

Stabilized selects the correct branch on every seed for all perfect, mild, and stochastic adversaries, while Standard collapses on the same settings (\(\approx 16\%\) to \(39\%\) correct; Figure~\ref{fig:partial-adversary}(a)). The recovered adversary trust tracks corruption strength (Figure~\ref{fig:partial-adversary}(b)): it is strongly negative for a deterministic anti-oracle, proportionally weaker for milder \(\beta_4\), and near zero or mildly positive once flips are stochastic. A consistent adversary is thus inverted in proportion to its strength, whereas one that flips only some of the time is down-weighted rather than inverted. Trust learning curves are in Appendix~\ref{app:synthetic-sensitivity}.

\begin{figure}[t]
    \centering
    \begin{subfigure}[b]{0.48\linewidth}
        \centering
        \includegraphics[width=\linewidth]{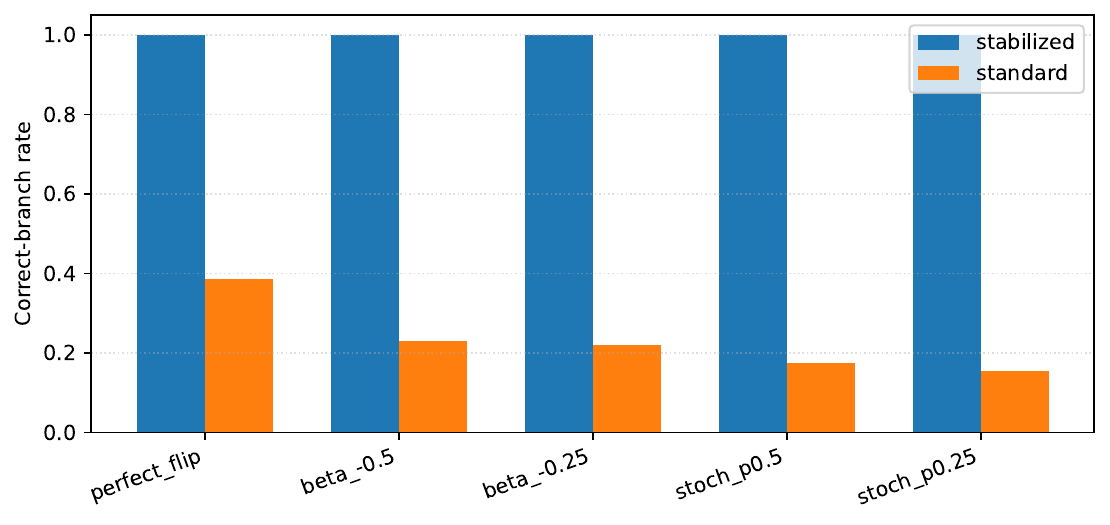}
        \caption{Correct-branch rate}
        \label{fig:partial-adversary-branch}
    \end{subfigure}\hfill
    \begin{subfigure}[b]{0.48\linewidth}
        \centering
        \includegraphics[width=\linewidth]{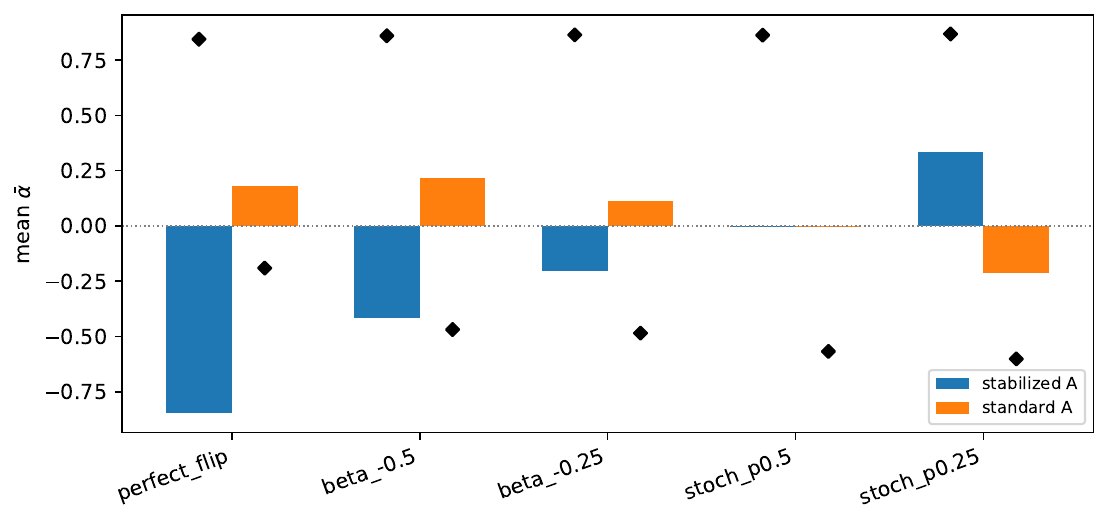}
        \caption{Recovered trust}
        \label{fig:partial-adversary-trust}
    \end{subfigure}
    \caption{Correct-branch rate~(a) and recovered trust~(b) under partial and stochastic adversaries on shared-head 3R1A, for Stabilized and Standard, over \(200\) seeds, with \(n=500\). Panel~(b) reports the mean learned \(\bar\alpha_A\) (bars) and the mean \(\bar\alpha_R\) (diamond markers).}
    \label{fig:partial-adversary}
\end{figure}

\subsection{End-to-end policy learning}
\label{sec:policy-experiments}

In this section, we integrate TTP into the PEBBLE loop~\citep{pebble} on four continuous-control tasks. The synthetic diagnostics showed when branch selection and trust separation succeed; here we evaluate whether those dynamics persist, and whether they improve policies, under simulated multi-expert B-Pref feedback.

\paragraph{Domains.}
We evaluate on four benchmark environments covering locomotion and manipulation, as shown in Figure~\ref{fig:experiment_environments}. \textsc{DMControl} Cheetah-Run tests long-horizon locomotion with dense returns. \textsc{MetaWorld} Door-Open-v2 requires contact-rich manipulation. \textsc{MetaWorld} Sweep-Into-v2 provides sparse success signals and is particularly fragile under corrupted preferences. \textsc{DMControl} Walker-Walk is used to study feedback efficiency at budgets of 500, 1{,}000, 2{,}000, 3{,}000 and 5{,}000 pairwise comparisons.    

\begin{figure}[!t]
    \centering

    \begin{subfigure}[b]{0.48\linewidth}
        \centering
        \includegraphics[width=0.62\linewidth]{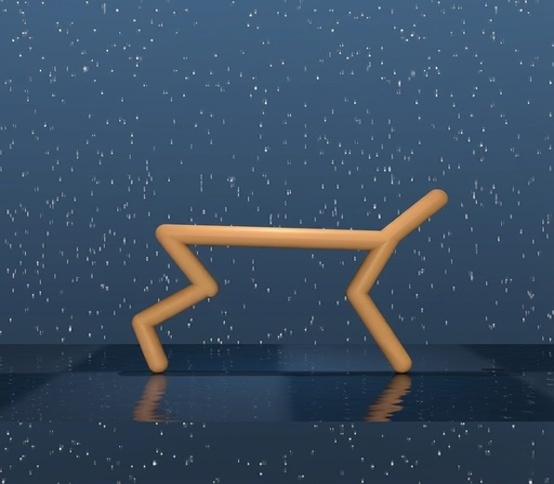}
        \caption{\textsc{DMControl} Cheetah-Run}
        \label{fig:env_cheetah}
    \end{subfigure}\hfill
    \begin{subfigure}[b]{0.48\linewidth}
        \centering
        \includegraphics[width=0.62\linewidth]{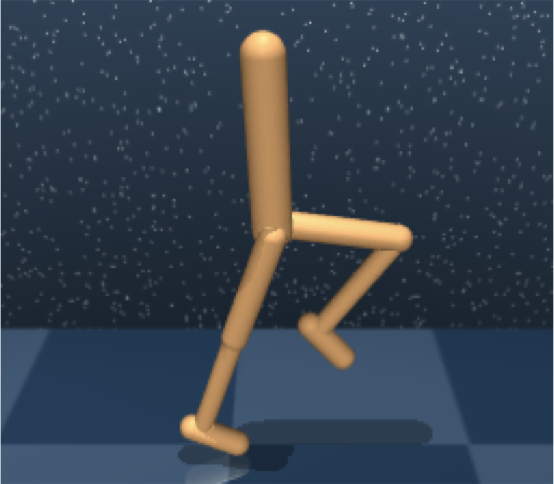}
        \caption{\textsc{DMControl} Walker-Walk}
        \label{fig:env_walker}
    \end{subfigure}

    \vspace{2mm}

    \begin{subfigure}[b]{0.48\linewidth}
        \centering
        \includegraphics[width=0.62\linewidth]{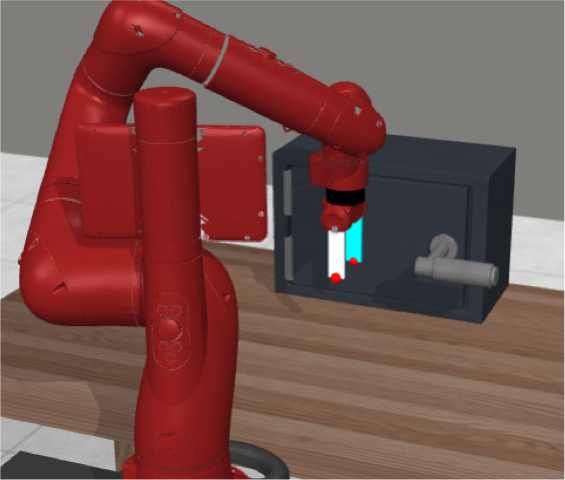}
        \caption{\textsc{MetaWorld} Door-Open-v2}
        \label{fig:env_door}
    \end{subfigure}\hfill
    \begin{subfigure}[b]{0.48\linewidth}
        \centering
        \includegraphics[width=0.62\linewidth]{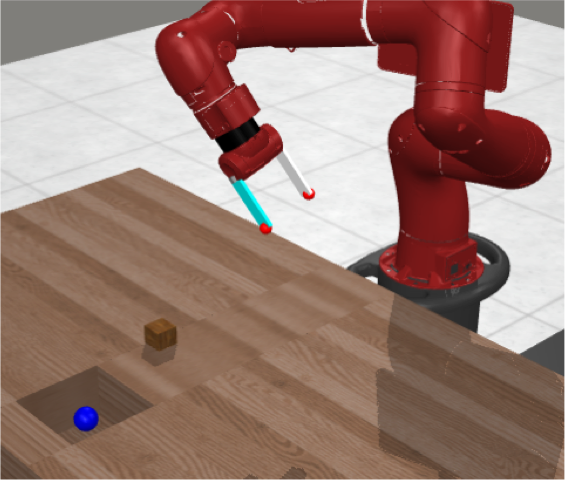}
        \caption{\textsc{MetaWorld} Sweep-Into-v2}
        \label{fig:env_sweep}
    \end{subfigure}

    \caption{Experimental environments used in this work. We evaluate on two locomotion domains from \textsc{DMControl} (top row) and two manipulation domains from \textsc{MetaWorld} (bottom row).}
    \label{fig:experiment_environments}
\end{figure}

\paragraph{Multi-expert configuration.}
Following B-Pref~\citep{lee2021bpref}, we simulate heterogeneous expert feedback using $K=4$ experts in the three benchmark environments; due to computational limitations, we restrict the full end-to-end runs with multiple methods to this pool size. The Walker-Walk budget study uses $K=5$ to sweep expert mixtures. Each expert $k$ has ground-truth rationality $\beta_k \in \{-1, 0, 1\}$. 
The principal mixtures are \emph{adversarial} \(\beta=[1,1,1,-1]\) (25\% adversarial) and \emph{noisy} \(\beta=[1,1,1,0]\). Each expert labels \emph{disjoint} trajectory pairs, so robustness cannot rely on majority vote on identical pairs.

\paragraph{Implementation details.}
We implement TTP on top of PEBBLE~\citep{pebble} using Soft Actor-Critic (SAC)~\citep{haarnoja2018soft}. The reward model is an MLP; trust parameters \(\alpha_k\) are initialized at 0.01 and updated by gradient descent. Full hyperparameters and training procedures are in our \href{https://github.com/SeyedAmirHs00/tritrust-pbrl-code/}{GitHub repository}. All comparison curves are averaged over 10 random seeds with standard-error bands.

\paragraph{Baselines and metrics.}
We compare against: (i)~uniform-weight PEBBLE~\citep{pebble}; (ii)~RIME~\citep{rime}; (iii)~MCP~\citep{mcp}; and (iv)~oracle SAC~\citep{haarnoja2018soft}. MetaWorld uses success rate; \textsc{DMControl} uses true episode return. The Walker-Walk heatmap (Figure~\ref{fig:heatmap_walker}) reports TTP across mixtures and budgets.

RIME and MCP target sample-level corruption, not per-annotator systematic inversion, so under \(\beta_k=-1\) they are strong PEBBLE-lineage baselines operating outside their design intent. The comparison asks whether annotator-aware signed trust helps where filtering alone is not enough. We do not claim an exhaustive survey of every recent robust PBRL or classical crowd estimator~\citep{sun2025rethinking,xu2026human,gong2025adaptive,huang2025trend,zhang2025diverging,dawid1979maximum,raykar2010learning,whitehill2009whose}; those lines do not, as stated, model per-annotator sign inversion on disjoint multi-expert feedback. 

\subsubsection{MetaWorld Sweep-Into-v2}

\paragraph{Adversarial experts ($\beta=[1,1,1,-1]$)}
On this sparse-success manipulation task, \(25\%\) adversarial feedback is enough to stall the baselines: PEBBLE and MCP never leave the near-zero regime, and RIME escapes it but plateaus at roughly half of TTP's final success rate (Figure~\ref{fig:eval_sweepinto_adv}). TTP alone approaches oracle SAC (about \(85\%\) against \(100\%\)). Its trusts separate early, positive for the three reliable experts and negative for the adversary (Figure~\ref{fig:reward_alphas_sweep_adv}), consistent with inversion correcting the reward before the policy drifts into hard-to-recover regions.

\begin{figure}[!htbp]
    \centering
    \includegraphics[width=0.75\linewidth]{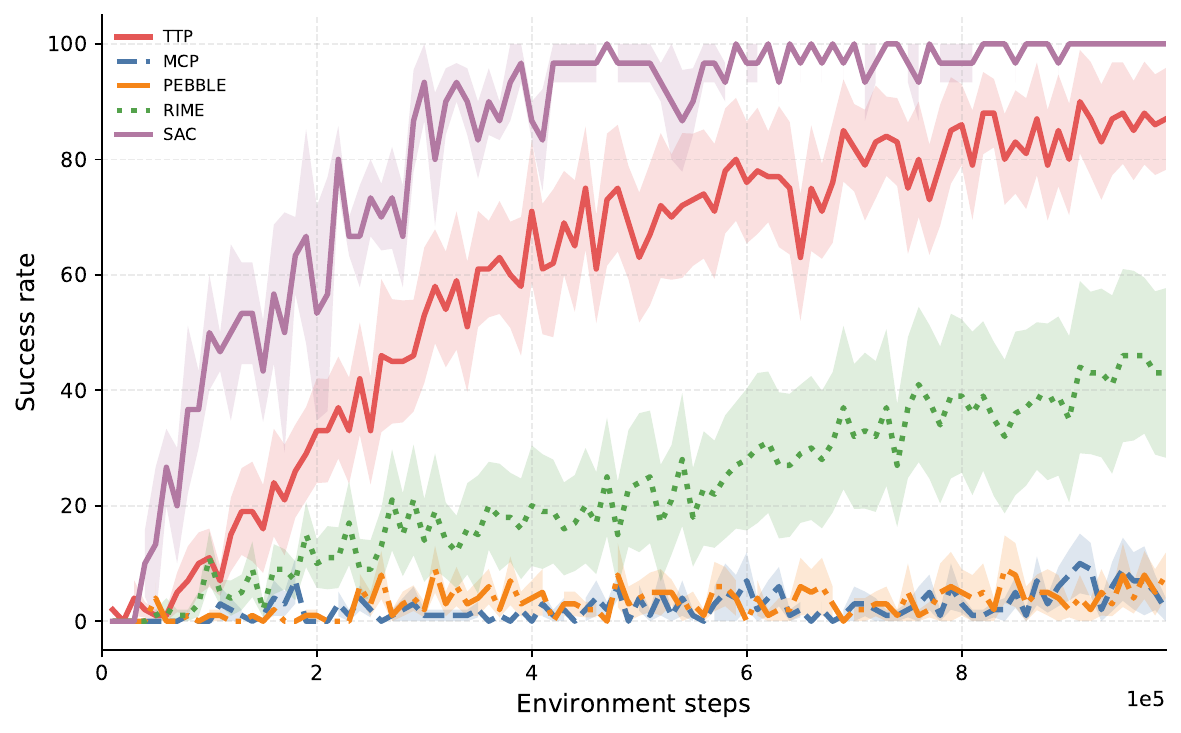}
    \caption{Evaluation success rate for MetaWorld Sweep-Into-v2 under adversarial conditions (\(\beta=[1,1,1,-1]\)).}
    \label{fig:eval_sweepinto_adv}
\end{figure}

\begin{figure}[!htbp]
    \centering

    \begin{subfigure}[b]{0.48\textwidth}
        \centering
        \includegraphics[width=\linewidth]{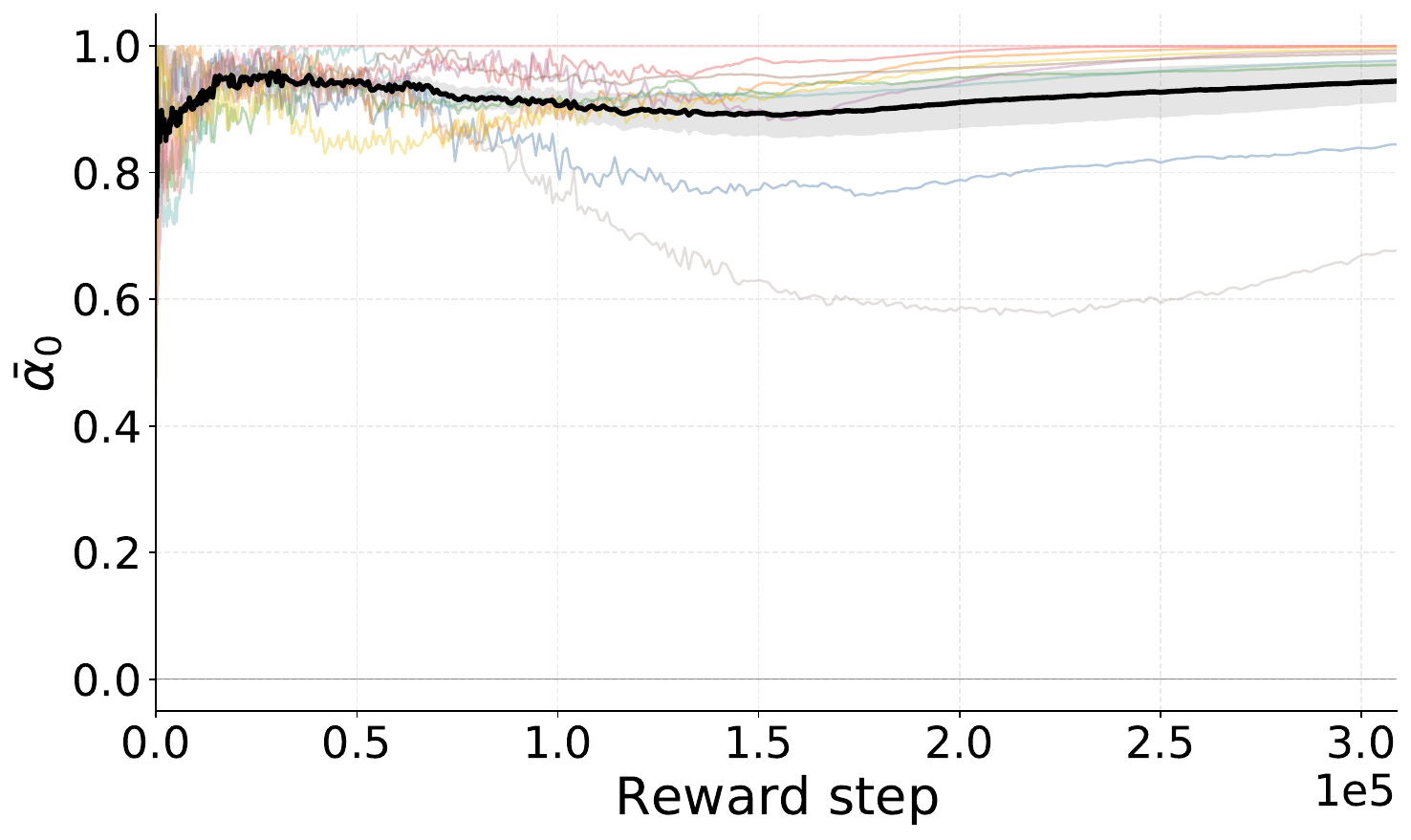}
        \caption{Expert 0 (\(\beta_0=1\), reliable)}
        \label{fig:alpha0_sweep_adv}
    \end{subfigure}\hfill
    \begin{subfigure}[b]{0.48\textwidth}
        \centering
        \includegraphics[width=\linewidth]{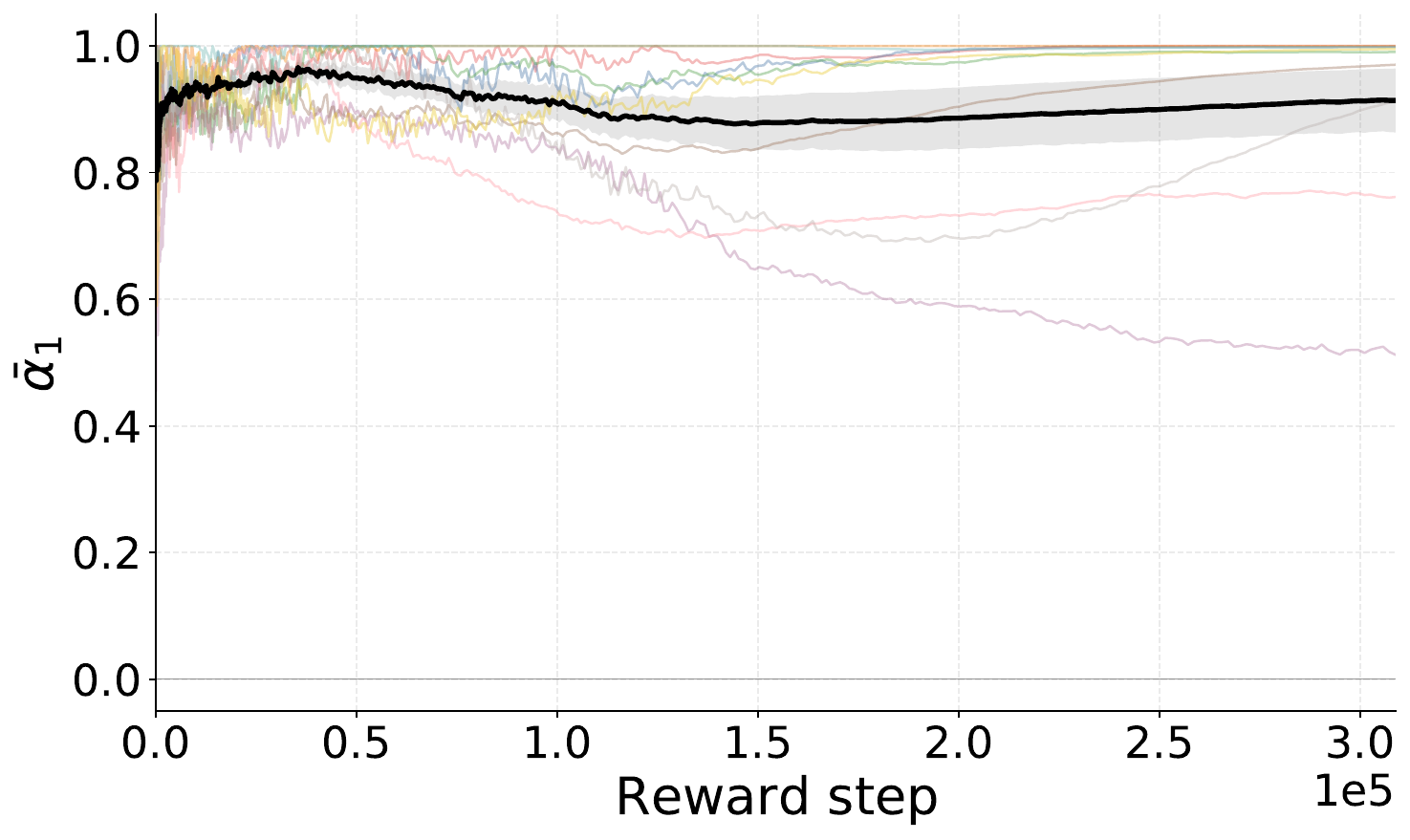}
        \caption{Expert 1 (\(\beta_1=1\), reliable)}
        \label{fig:alpha1_sweep_adv}
    \end{subfigure}

    \vspace{1ex}

    \begin{subfigure}[b]{0.48\textwidth}
        \centering
        \includegraphics[width=\linewidth]{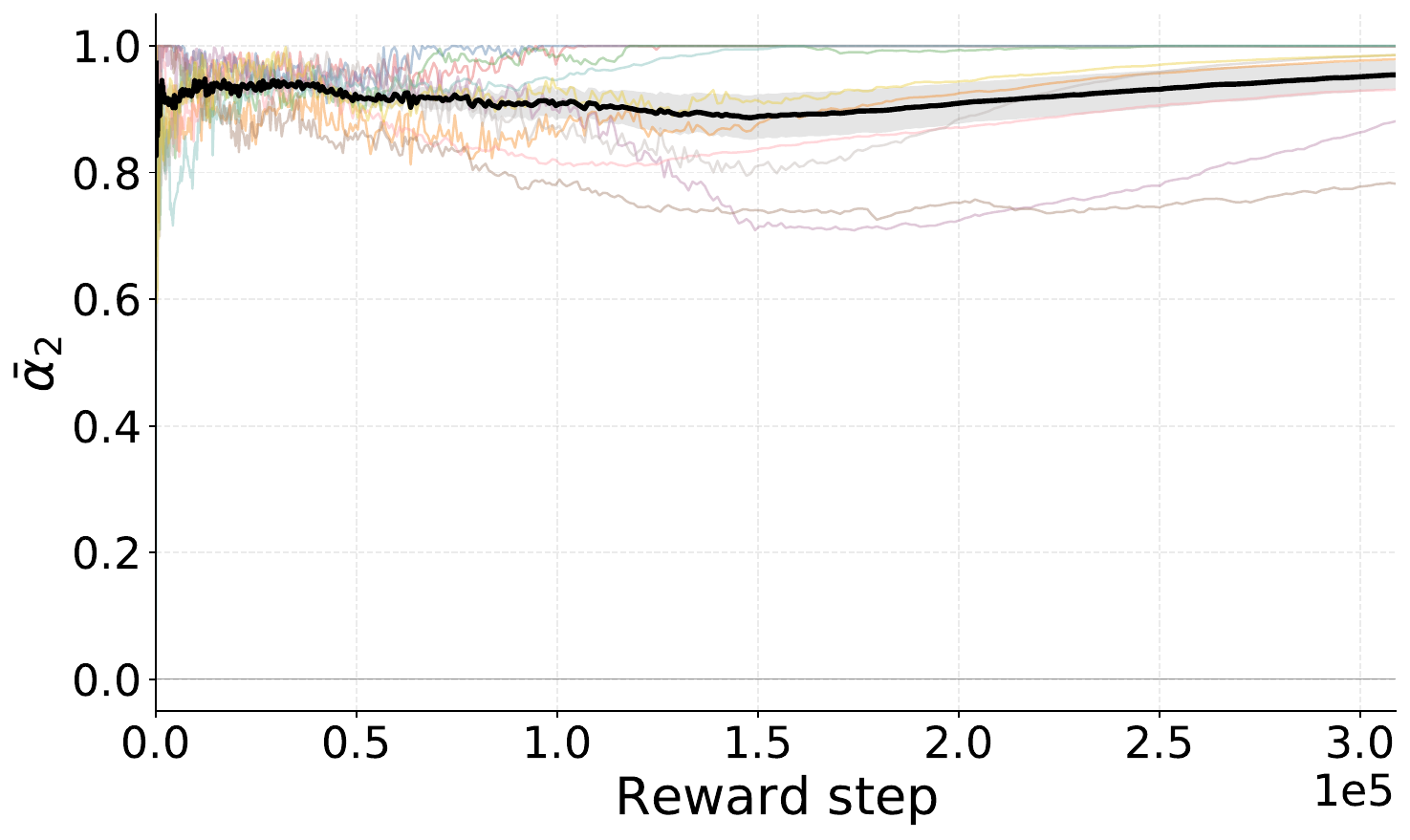}
        \caption{Expert 2 (\(\beta_2=1\), reliable)}
        \label{fig:alpha2_sweep_adv}
    \end{subfigure}\hfill
    \begin{subfigure}[b]{0.48\textwidth}
        \centering
        \includegraphics[width=\linewidth]{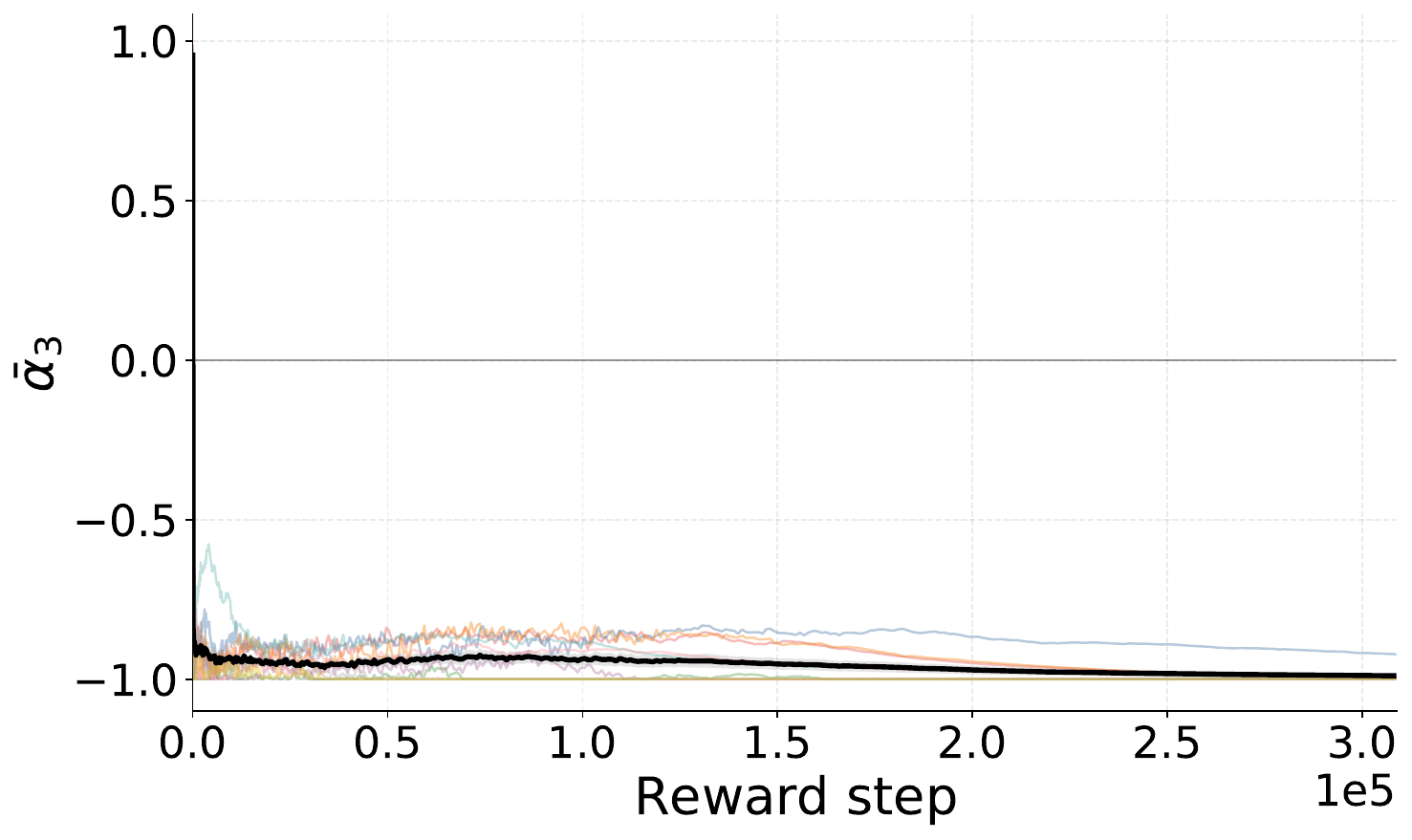}
        \caption{Expert 3 (\(\beta_3=-1\), adversarial)}
        \label{fig:alpha3_sweep_adv}
    \end{subfigure}

    \caption{Learned trust trajectories on Sweep-Into-v2 under adversarial conditions.}
    \label{fig:reward_alphas_sweep_adv}
\end{figure}

\paragraph{Noisy experts (\(\beta=[1,1,1,0]\)).}
Replacing the adversary with an uninformative expert lets every method learn and narrows the gap (Figure~\ref{fig:eval_sweepinto_noisy}). TTP still finishes closest to oracle SAC, overtaking RIME after a slower start, while PEBBLE and MCP stay below both. The noisy expert's trust stays near zero (Figure~\ref{fig:reward_alphas_sweep_noisy}), as Section~\ref{sec:alpha-analysis} predicts for labels uncorrelated with the reward, so it contributes little to the reward update.

\begin{figure}[!htbp]
    \centering
    \includegraphics[width=0.75\linewidth]{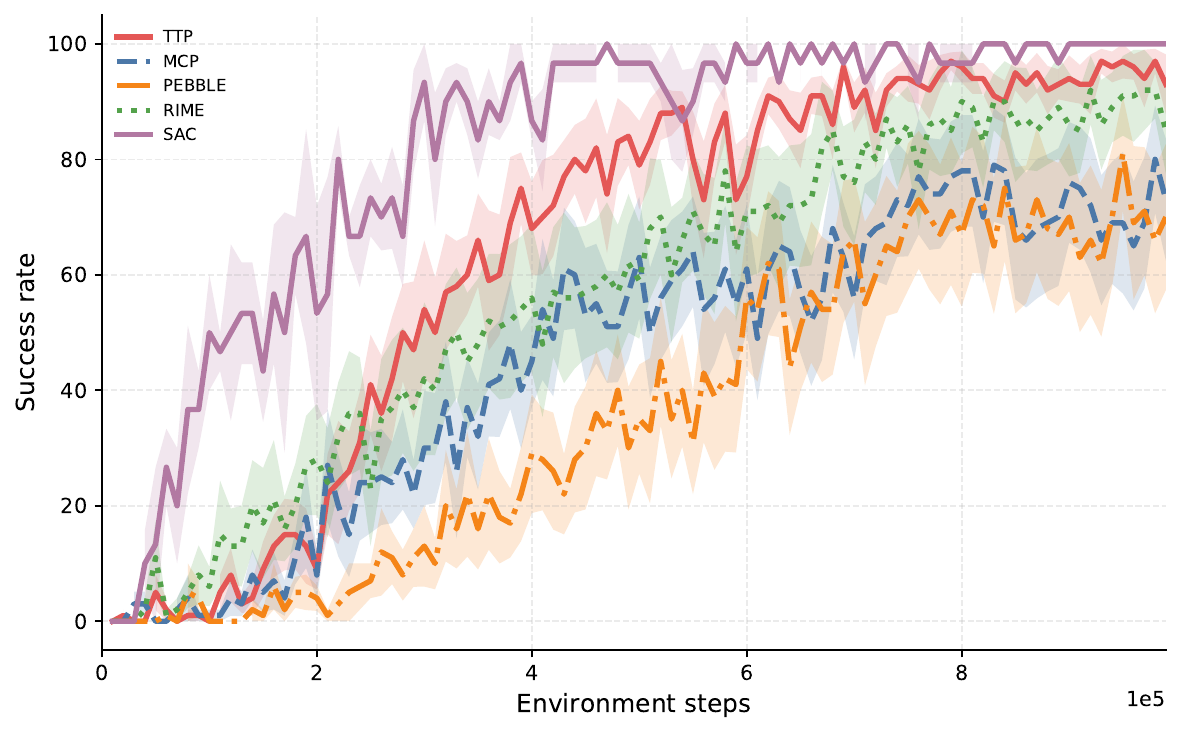}
    \caption{Evaluation success rate for MetaWorld Sweep-Into-v2 under noisy conditions (\(\beta=[1,1,1,0]\)).}
    \label{fig:eval_sweepinto_noisy}
\end{figure}

\begin{figure}[!htbp]
    \centering

    \begin{subfigure}[b]{0.48\textwidth}
        \centering
        \includegraphics[width=\linewidth]{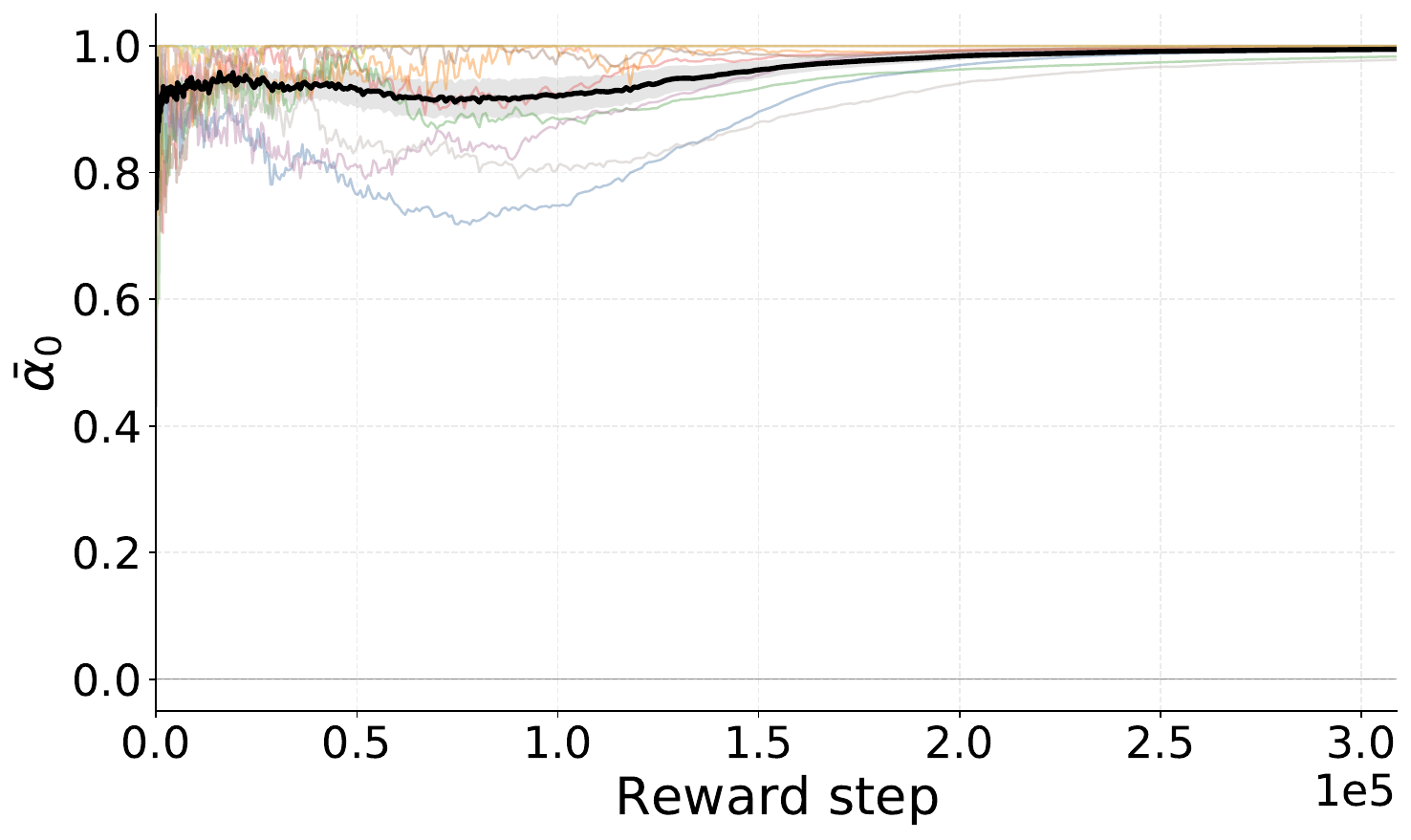}
        \caption{Expert 0 (\(\beta_0=1\), reliable)}
        \label{fig:alpha0_sweep_noisy}
    \end{subfigure}\hfill
    \begin{subfigure}[b]{0.48\textwidth}
        \centering
        \includegraphics[width=\linewidth]{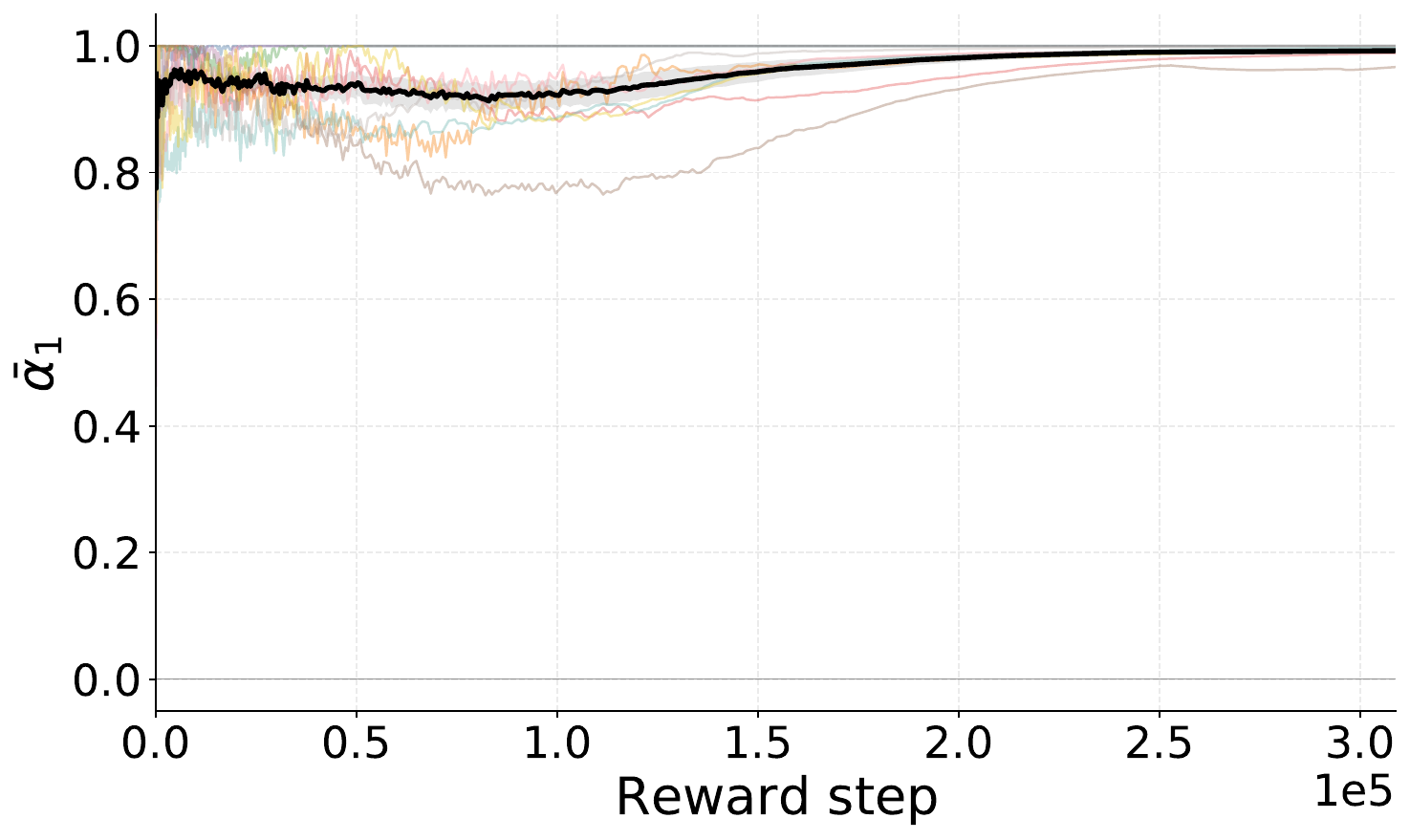}
        \caption{Expert 1 (\(\beta_1=1\), reliable)}
        \label{fig:alpha1_sweep_noisy}
    \end{subfigure}

    \vspace{1ex}

    \begin{subfigure}[b]{0.48\textwidth}
        \centering
        \includegraphics[width=\linewidth]{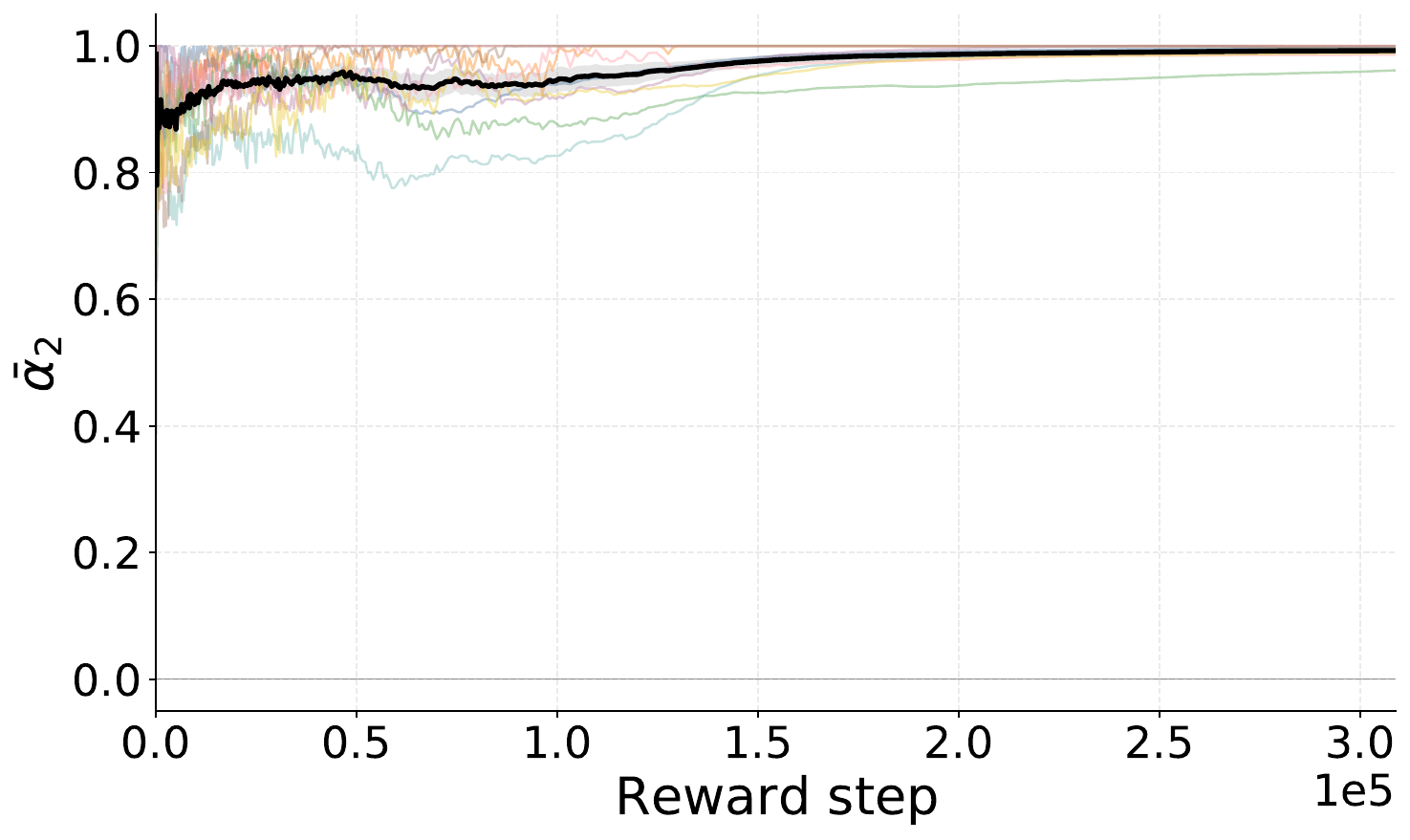}
        \caption{Expert 2 (\(\beta_2=1\), reliable)}
        \label{fig:alpha2_sweep_noisy}
    \end{subfigure}\hfill
    \begin{subfigure}[b]{0.48\textwidth}
        \centering
        \includegraphics[width=\linewidth]{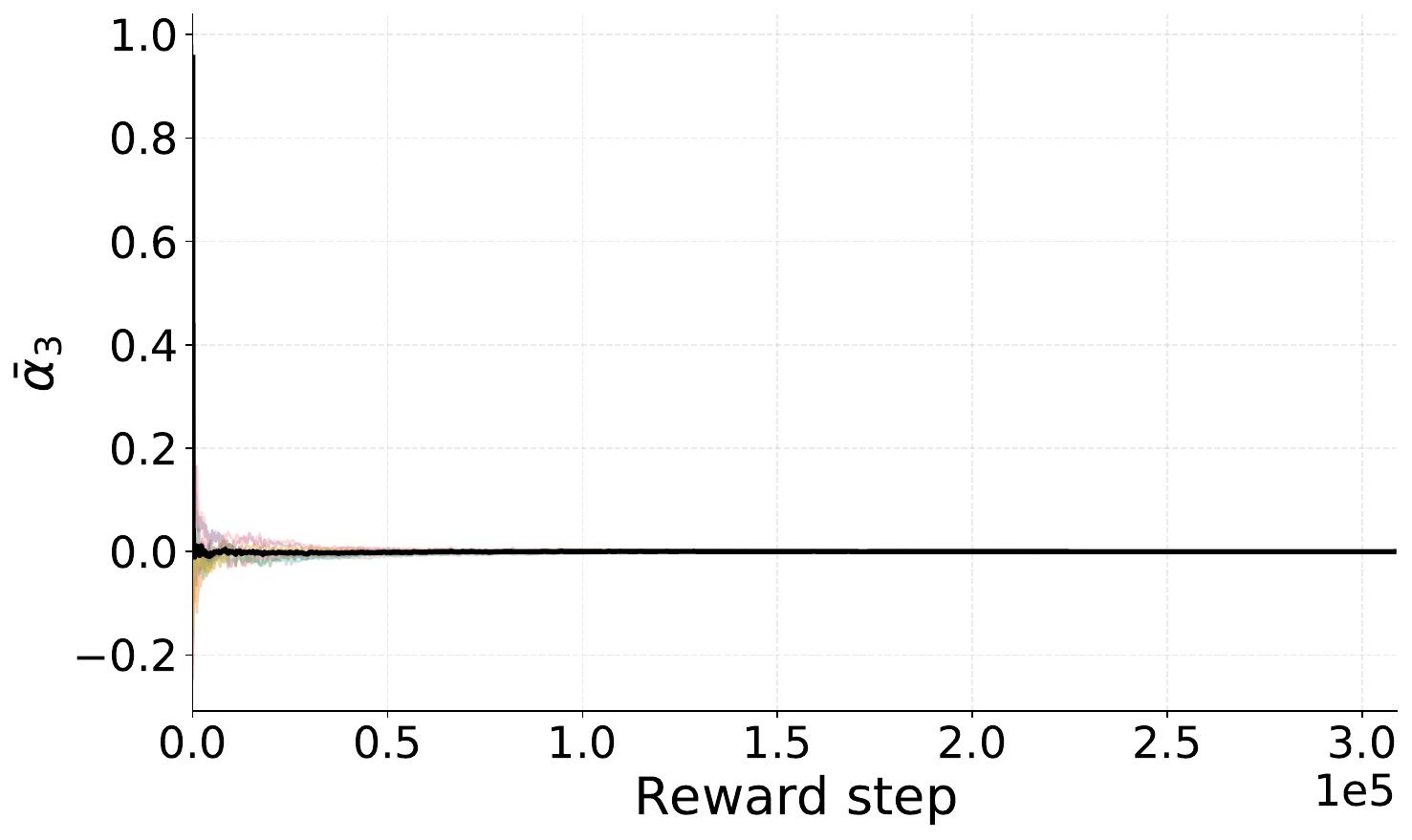}
        \caption{Expert 3 (\(\beta_3=0\), noisy)}
        \label{fig:alpha3_sweep_noisy}
    \end{subfigure}

    \caption{Learned trust trajectories on Sweep-Into-v2 under noisy conditions.}
    \label{fig:reward_alphas_sweep_noisy}
\end{figure}

\subsubsection{MetaWorld Door-Open-v2}

\paragraph{Adversarial experts (\(\beta=[1,1,1,-1]\)).}
We next evaluate robustness to adversarial preference feedback on the \textsc{MetaWorld} Door-Open-v2 task. As in the previous experiments, we simulate four experts, one of which provides systematically adversarial comparisons.

The resulting learning curves are shown in Figure~\ref{fig:eval_door_adv}. TTP maintains high success rates despite adversarial corruption, whereas PEBBLE degrades sharply. The larger performance gap between TTP and the baselines in this environment highlights the importance of explicitly correcting biased supervision, particularly for tasks that require precise and sequential interactions to succeed.

The trust parameters learned by TTP assign positive weight to the reliable experts and negative weight to the adversarial expert, with per-expert trust trajectories closely mirroring the pattern already shown for Sweep-Into-v2 (Figure~\ref{fig:reward_alphas_sweep_adv}); we omit the repeated plot here. This consistency suggests that the trust mechanism captures an intrinsic property of the experts, namely the direction of their correlation with true task progress, rather than overfitting to a specific environment.

\begin{figure}[!htbp]
    \centering
    \includegraphics[width=0.75\linewidth]{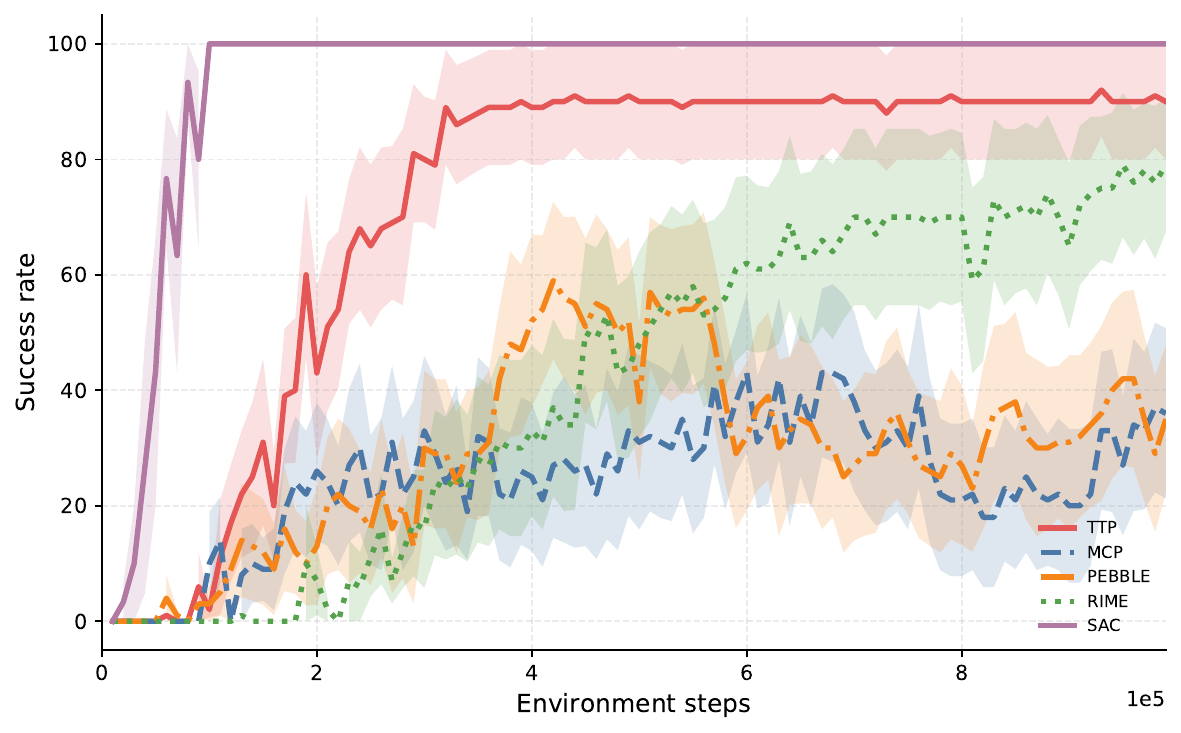}
    \caption{Evaluation success rate for MetaWorld Door-Open-v2 under adversarial conditions (\(\beta=[1,1,1,-1]\)).}
    \label{fig:eval_door_adv}
\end{figure}

\paragraph{Noisy experts (\(\beta=[1,1,1,0]\)).}
We evaluate the Door-Open-v2 task under noisy but non-adversarial feedback, where one of four experts provides uninformative comparisons. The learning curves in Figure~\ref{fig:eval_door_noisy} show that TTP maintains strong performance. As in Sweep-Into-v2 (Figure~\ref{fig:reward_alphas_sweep_noisy}), the trust parameter for the noisy expert remains near zero while the reliable experts receive positive trust; we omit the repeated plot here.

\begin{figure}[!htbp]
    \centering
    \includegraphics[width=0.75\linewidth]{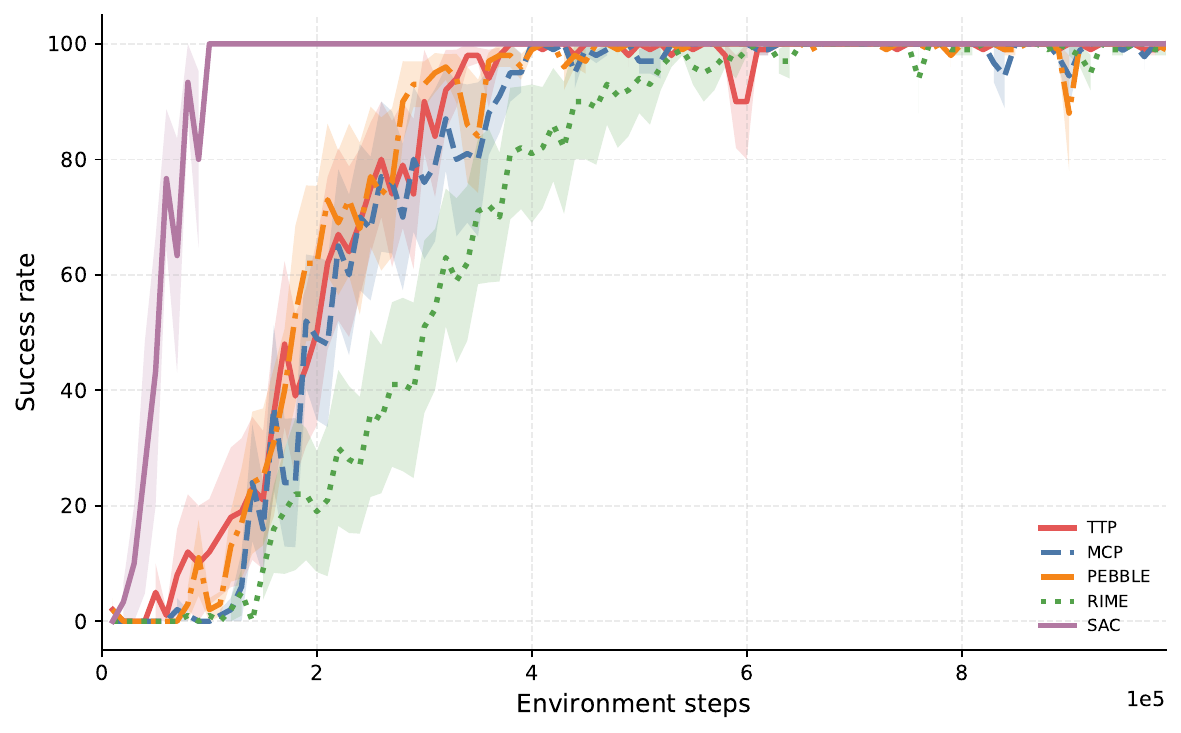}
    \caption{Evaluation success rate for MetaWorld Door-Open-v2 under noisy conditions (\(\beta=[1,1,1,0]\)).}
    \label{fig:eval_door_noisy}
\end{figure}

\subsubsection{\textsc{DMControl} Cheetah-Run}

\paragraph{Adversarial experts (\(\beta=[1,1,1,-1]\)).}
Dense environment rewards do not protect the baselines: adversarial preferences still bias their learned reward, and they plateau far below TTP (about \(800\) return for TTP against at most \(570\), for RIME; Figure~\ref{fig:eval_cheetah_adv}). TTP tracks oracle SAC from mid-training onward. Its trusts show the same sign pattern as on Sweep-Into-v2 (Figure~\ref{fig:reward_alphas_sweep_adv}); we omit the repeated plot.

\begin{figure}[!htbp]
    \centering
    \includegraphics[width=0.75\linewidth]{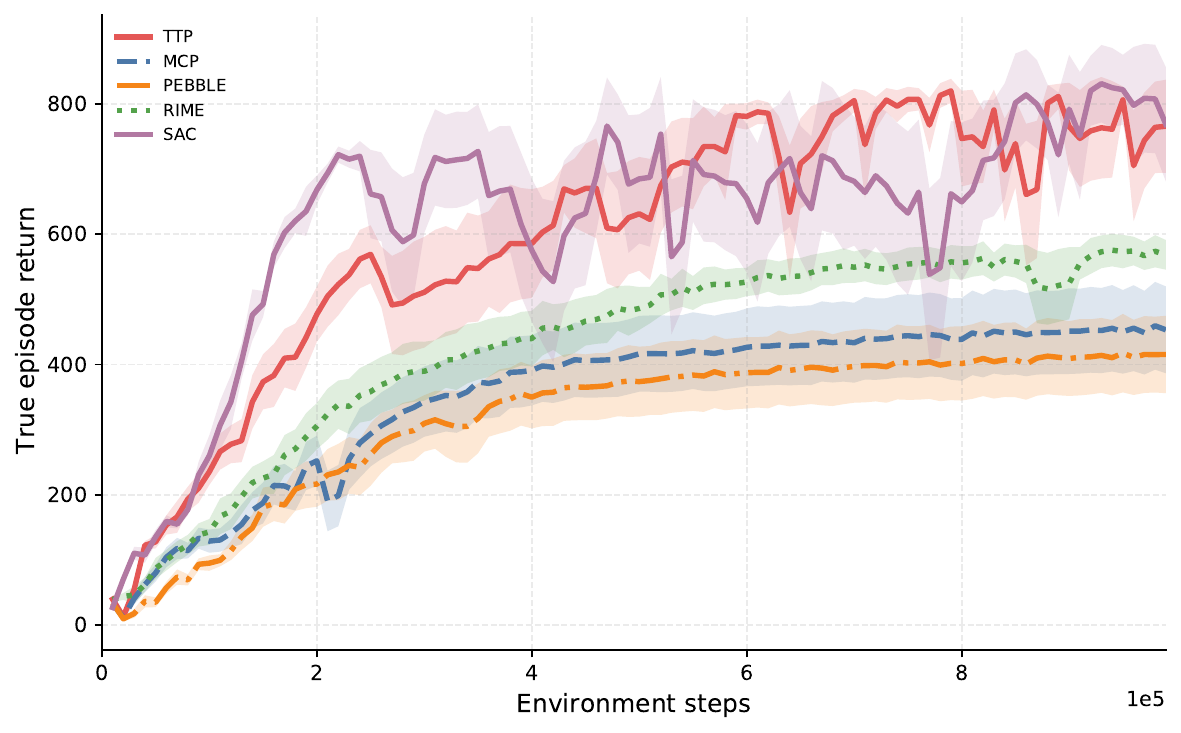}
    \caption{True episode reward for \textsc{DMControl} Cheetah-Run under adversarial conditions (\(\beta=[1,1,1,-1]\)).}
    \label{fig:eval_cheetah_adv}
\end{figure}

\paragraph{Noisy experts (\(\beta=[1,1,1,0]\)).}
Without a systematic adversary, the margin over the baselines is smaller and appears only in the second half of training, after which TTP ends level with oracle SAC (Figure~\ref{fig:eval_cheetah_noisy}). As on Sweep-Into-v2, the noisy expert's trust stays near zero; we omit the repeated plot. 

\begin{figure}[!htbp]
    \centering
    \includegraphics[width=0.75\linewidth]{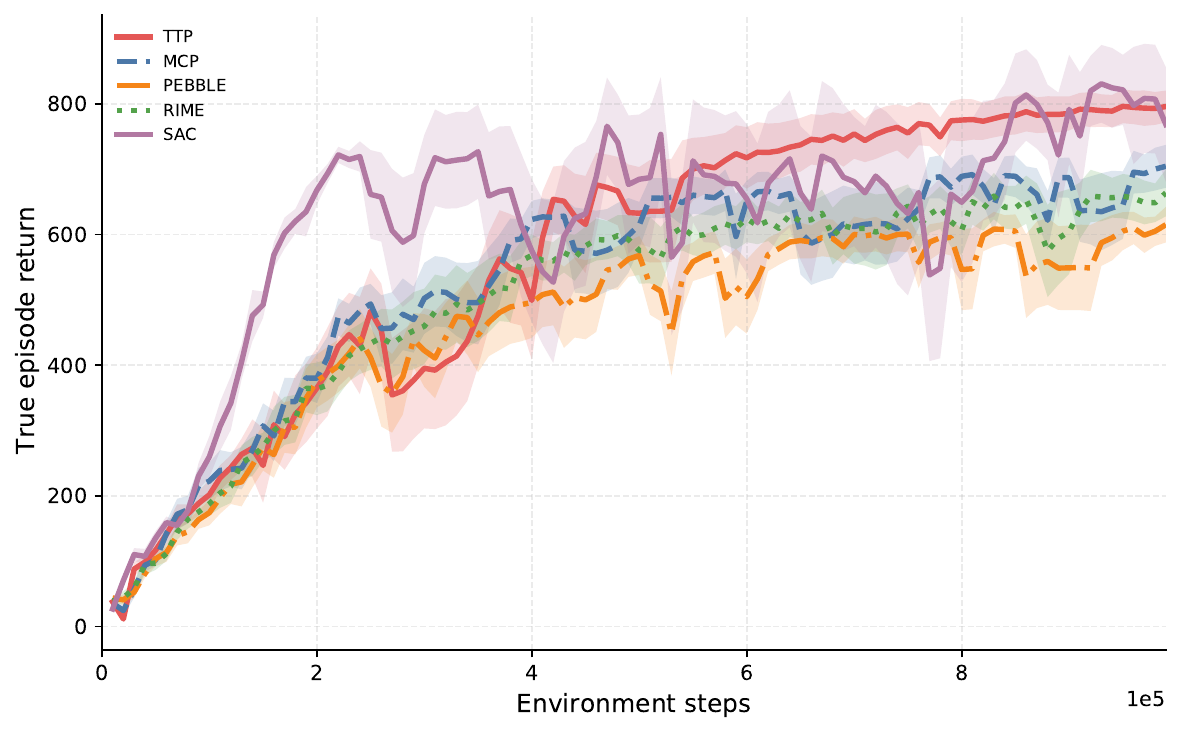}
    \caption{True episode reward for \textsc{DMControl} Cheetah-Run under noisy conditions (\(\beta=[1,1,1,0]\)).}
    \label{fig:eval_cheetah_noisy}
\end{figure}

\subsubsection{Feedback Efficiency: \textsc{DMControl} Walker-Walk}

We study how learning performance scales with the amount of preference feedback in the \textsc{DMControl} Walker-Walk locomotion task. In this experiment, we simulate $K=5$ experts with heterogeneous reliability and vary both the total feedback budget and the composition of expert types. Each row in the heatmap corresponds to a different expert reliability configuration $\beta \in \{-1,0,1\}^5$, capturing mixtures of reliable, noisy, and adversarial experts. We evaluate five feedback budgets: 500, 1{,}000, 2{,}000, 3{,}000, and 5{,}000 comparisons, with 5 seeds per cell; the \(\pm\) values quoted below are standard deviations across those seeds.

Figure~\ref{fig:heatmap_walker} reports final Walker-Walk performance as a function of feedback budget and expert reliability configuration. Performance varies substantially across expert mixtures, indicating that the effect of trust learning depends strongly on the structure of the feedback. We observe that:

\begin{itemize}
    \item \textbf{The reliable fraction determines how much feedback is needed.} When most experts are reliable ($\beta=[1,1,1,1,1]$ or $[1,1,1,1,0]$), performance is already comparatively high at 500 comparisons and grows steadily, largely saturating by 2{,}000 to 3{,}000. Mixtures with fewer reliable experts start much lower at 500 comparisons (e.g., $\beta=[1,1,0,0,-1]$, two reliable, two noisy, one adversarial, starts at roughly half its eventual value) but, as long as reliable experts remain a majority, still reach comparably high performance by 2{,}000 to 3{,}000.

    \item \textbf{Noisy experts ($\beta=0$) primarily add variance.} Replacing reliable experts with noisy ones slows convergence and increases instability at low budgets, but performance recovers with additional feedback as the noise averages out; e.g., $\beta=[1,1,0,0,0]$ is well below the ceiling at 500 to 1{,}000 comparisons but reaches high performance by 3{,}000 to 5{,}000.

    \item \textbf{An adversarial expert is worth more than a noisy one.} Replacing a noisy expert with an adversarial one holds the reliable count fixed yet raises low-budget performance in three of the four matched pairs available (e.g., $\beta=[1,1,0,0,-1]$ vs.\ $[1,1,0,0,0]$: $533$ vs.\ $263$ at 500 comparisons; $+98$ on average at 500 and $+36$ at 1{,}000, closing by 3{,}000). Signed trust inverts a systematic adversary into usable signal, whereas a $\beta=0$ expert carries none to recover.


    \item \textbf{Feedback quantity cannot substitute for feedback quality.} The benefit of additional feedback depends on its source: further comparisons help once reliable experts form a majority, but when adversarial feedback dominates the mixture, additional data does not yield the improvement observed elsewhere.
\end{itemize}

\begin{figure}[!htbp]
    \centering
    \includegraphics[width=\linewidth]{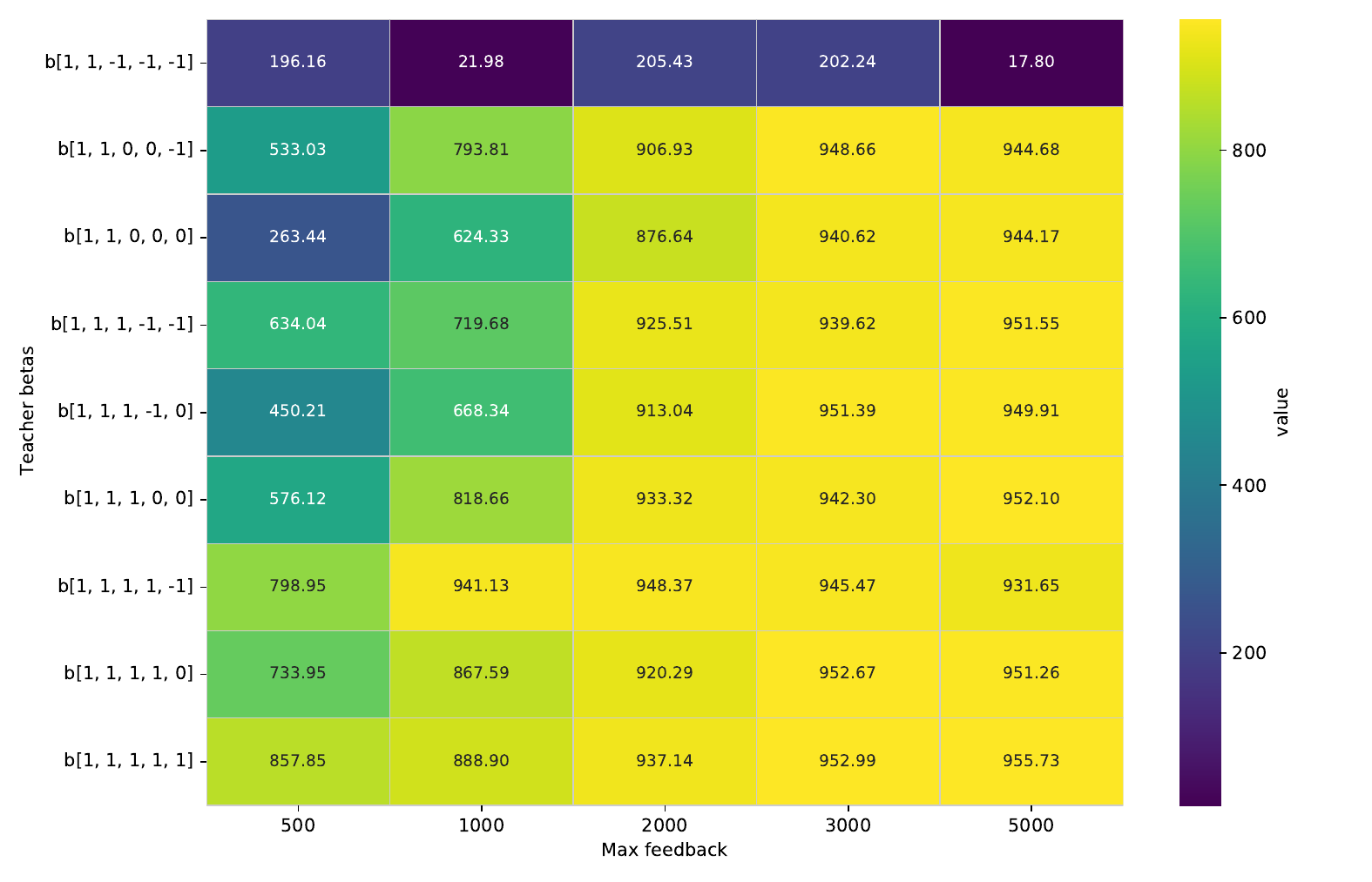}
    \caption{Performance heatmap for \textsc{DMControl} Walker-Walk showing final episode reward as a function of feedback amount and expert reliability configuration.}
    \label{fig:heatmap_walker}
\end{figure}

\subsubsection{Ablation of practical enhancements}
\label{sec:enhancement-ablation-policy}

\begin{figure}[!htbp]
    \centering
    \includegraphics[width=0.75\linewidth]{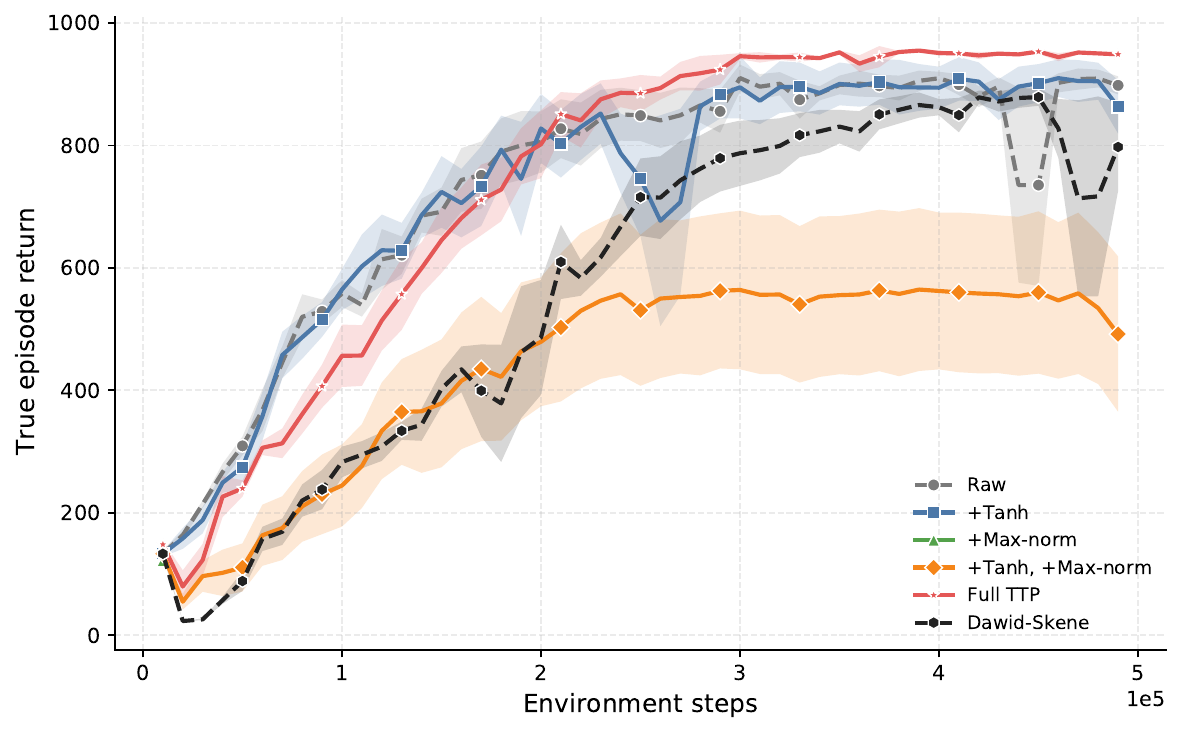}
    \caption{Ablation of practical enhancements in the PEBBLE loop on \textsc{DMControl} Walker-Walk, \(\beta=[1,1,0,0,-1]\) (\(K=5\); two reliable, two noisy, one adversarial). Components of Section~\ref{sec:cross-entropy-weight} are enabled cumulatively; a Dawid--Skene flip-channel learner (Eq.~\eqref{eq:related-correctness-channel}) on the same reward model is shown for reference. Shaded bands show variability across seeds. Only the full combination is both high-performing and stable; Dawid--Skene learns more slowly and finishes lower.}
    \label{fig:ablation_learning_curve}
\end{figure}

We enable the components of Section~\ref{sec:cross-entropy-weight} one at a time on \textsc{DMControl} Walker-Walk under the mixture \(\beta=[1,1,0,0,-1]\) (\(K=5\); two reliable, two noisy, one adversarial), holding the two-path split (Eqs.~\eqref{eq:loss-R}--\eqref{eq:loss-A}) fixed: \emph{Raw} uses the scalar trust unmodified, \emph{+Tanh} bounds it, \emph{+Tanh, +Max-norm} rescales it, and \emph{Full TTP} adds the detached weights \(w_k\). Figure~\ref{fig:ablation_learning_curve} shows that only the full combination works. Full TTP is both the highest and the steadiest, holding roughly \(950\) return from \(3\times10^5\) steps onward. 

The reason is that the two components do different jobs on the reward path, and neither is complete alone. Max-normalization sets the \emph{scale}: dividing by \(\max_{k'}|\tilde\alpha_{k'}|\) keeps at least one expert at unit magnitude, so \(\nabla R\) does not vanish at initialization. But scale applies to everyone, so the two noisy experts push on the reward as hard as the reliable ones, and the \(550\) plateau is a reward model learning from that noise. The weights \(w_k\) supply the missing \emph{selectivity}, damping experts whose trust has collapsed toward zero while leaving the adversary active through the sign of \(\bar\alpha_k\).

Figure~\ref{fig:ablation_learning_curve} also includes a Dawid--Skene baseline~\citep{dawid1979maximum}, which keeps the shared reward network but replaces signed trust with the correctness channel of Eq.~\eqref{eq:related-correctness-channel}: each expert reports the model-consistent label with a learned probability \(q_k\), and reverses it otherwise. In principle, this channel can also represent all three expert types, with \(q_k\approx1\), \(\tfrac12\), and \(0\) for reliable, noisy, and adversarial experts. In practice, it underperforms Full TTP throughout training. It learns markedly more slowly, staying near \(400\) return at \(2\times10^5\) steps when Full TTP is above \(800\). It then plateaus around \(850\), below Full TTP's \(950\), and becomes unstable late in training, dropping to roughly \(700\) near the end.

\subsection{Limitations and Future Directions}

The experiments establish a coherent picture. Under a reliable majority, signed trust separates expert types as the gradient analysis predicts, synthetic diagnostics show when the correct reward branch is selected, and on three control domains that same mechanism yields stronger policies than uniform PEBBLE, especially under adversarial mixtures, with clear gains over sample-level robust baselines when the curves separate.

Two scope conditions bound it. First, inversion needs a reliable majority: if adversaries outnumber reward-aligned experts, optimization can lock onto the flipped branch and swap expert roles (Table~\ref{tab:synthetic-branch}). Negative trust is then still ``correct'' relative to \(\hat R\), but \(\hat R\) itself is wrong for policy learning. Second, Theorem~\ref{thm:identifiability} assumes comparison overlap, whereas the PEBBLE runs use disjoint pairs coupled only through the shared network \(\hat r_\theta\). That mismatch is intentional for crowdsourcing realism, but our identifiability result does not cover it. Parameter sharing couples the experts' datasets as an optimization mechanism (Remark~\ref{rem:theory-disjoint}), but how much signal transfers between experts' trajectory sets depends on the reward model and the data.

Natural extensions of the model itself are also left for future work: context-dependent trust when reliability varies by state or difficulty, and using learned trusts for active expert selection rather than only for reward weighting.

\section{Conclusion}
\label{sec:conclusion}

This work studied interactive preference-based reinforcement learning with identifiable experts whose judgments may be reward-aligned, uninformative, or systematically anti-aligned. TriTrust-PBRL addresses this setting by introducing an expert-specific signed rationality parameter into the preference likelihood and learning these parameters jointly with a shared neural reward. Positive, near-zero, and negative rationalities respectively retain, attenuate, and invert an expert's feedback. The shared reward model couples observations across sources, so training remains possible from singly labeled queries even when expert datasets share no trajectory pairs, although how well signal transfers between them then depends on the reward model and the data. Joint estimation has a global sign ambiguity: simultaneously negating the reward and expert rationalities preserves the likelihood but reverses the policy objective. Under comparison overlap, max-normalization removes the continuous scaling ambiguity, leaving the reward determined up to an additive constant, which no reward difference depends on, and that global sign, while the stabilized two-path objective and zero-gap initialization select the intended sign in the reliable-majority regime. Synthetic experiments recover the intended reward in approximately \(99\%\) of such runs, and evaluations across four continuous-control domains show improved performance over PEBBLE, RIME, and MCP, particularly with adversarial feedback. The results support source-aware reward learning when expert tendencies persist across comparisons. The remaining dependence on a reliable majority limits robustness under adversarial dominance, while the use of a single global rationality per expert does not capture context-dependent expertise. Resolving reward orientation through external anchors and extending trust estimation to context-dependent or actively selected feedback are important directions for future work.

\paragraph{Code and data availability.} An implementation of TTP, the synthetic-diagnostic scripts, and the training configurations used for the PEBBLE experiments are available at \url{https://github.com/SeyedAmirHs00/tritrust-pbrl-code/}, with a project page at \url{https://seyedamirhs00.github.io/tritrust-pbrl/}. All environments used for evaluation (MetaWorld, DeepMind Control) are publicly available benchmarks; no new experimental data beyond simulated preference labels were collected.

\subsubsection*{Broader Impact Statement}
TriTrust-PBRL is designed to improve the robustness of preference-based reinforcement learning when feedback comes from heterogeneous sources, including uninformative or systematically anti-aligned annotators. Two limitations should be considered before deployment. First, signed trust assumes a single intended reward and a reliable majority: if systematically opposed annotators outnumber reward-aligned ones, the optimization can lock onto the inverted reward branch and induce opposite behavior. Second, treating persistent disagreement as negative trust is only appropriate when annotators share a common underlying objective; in settings with legitimate pluralistic values or minority viewpoints, inverting a dissenting source could override valid preferences rather than filter adversarial feedback.

\bibliographystyle{plainnat}
\bibliography{references}

\appendix

\refstepcounter{section}
\section*{Appendix~\thesection: Full identifiability proof}
\label{app:appendix}
\label{app:identifiability-proof}
\addcontentsline{toc}{section}{Appendix~\thesection: Full identifiability proof}

\begin{lemma}[Logistic identifiability]
\label{lem:logistic-ident}
For fixed $\alpha \neq 0$, the mapping $f: \Delta R \mapsto \sigma(\alpha \Delta R)$ is one-to-one. That is, if $\sigma(\alpha \Delta R_1) = \sigma(\alpha \Delta R_2)$, then $\Delta R_1 = \Delta R_2$.
\end{lemma}

\begin{proof}
Since $\sigma(z) = 1/(1+e^{-z})$ is strictly monotonic (its derivative $\sigma'(z) = \sigma(z)(1-\sigma(z)) > 0$ for all $z$), and $\alpha \neq 0$, the composition $\sigma(\alpha \cdot)$ is also strictly monotonic. Therefore, if $\sigma(\alpha \Delta R_1) = \sigma(\alpha \Delta R_2)$, we must have $\alpha \Delta R_1 = \alpha \Delta R_2$, which implies $\Delta R_1 = \Delta R_2$ since $\alpha \neq 0$.
\end{proof}

\begin{proof}[Proof of Theorem~\ref{thm:identifiability}]
We prove identifiability by characterizing all parameter transformations that leave the likelihood invariant.

\paragraph{Step 1: Only products $\alpha_k \Delta R_{ij}$ matter.}
For any pair $(\tau_i, \tau_j)$ in expert $k$'s dataset with $\alpha_k \neq 0$, the likelihood is
\[
P(y^{(k)}_{ij} = 1) = \sigma(\alpha_k (R(\tau_i) - R(\tau_j))).
\]
Suppose two parameter sets $(R, \{\alpha_k\})$ and $(R', \{\alpha_k'\})$ induce the \emph{same} probabilities for all observed pairs. Then, for every observed $(i,j,k)$,
\[
\sigma\!\big(\alpha_k (R(\tau_i) - R(\tau_j))\big)
=
\sigma\!\big(\alpha_k' (R'(\tau_i) - R'(\tau_j))\big).
\]
By Lemma~\ref{lem:logistic-ident} and $\alpha_k \neq 0$, this implies
\begin{equation}
\label{eq:prod-equality}
\alpha_k (R(\tau_i) - R(\tau_j))
=
\alpha_k' (R'(\tau_i) - R'(\tau_j))
\end{equation}
for all observed pairs $(\tau_i,\tau_j)$ of expert $k$.

\paragraph{Step 2: Per-expert equality of differences along labeled edges.}
Fix an arbitrary expert $k$ with $\alpha_k,\alpha_k'  \neq 0$. Restricting \eqref{eq:prod-equality} to this expert gives, for every pair $(\tau_i,\tau_j)\in D^{(k)}$,
\[
R'(\tau_i) - R'(\tau_j) = a_k\,(R(\tau_i) - R(\tau_j)),
\]
where $a_k := \alpha_k/\alpha_k' \neq 0$ is a constant that depends only on expert $k$ (and not on the particular pair). Thus, for every edge in the trajectory graph labeled by expert $k$, the difference in $R'$ is a fixed multiple of the difference in $R$.

\paragraph{Step 3: Expert overlap forces a common global scaling.}
Consider any two experts $k$ and $\ell$ with $\alpha_k,\alpha_\ell \neq 0$. By the expert connectedness assumption, there exists a path between $e_k$ and $e_\ell$ in the expert graph. For each edge along this path, the two incident experts share at least one comparison $(\tau_i,\tau_j)$ for which $R(\tau_i) \neq R(\tau_j)$. Applying \eqref{eq:prod-equality} to such a shared comparison for both experts and canceling the nonzero reward difference $R(\tau_i) - R(\tau_j)$ shows that
\[
\frac{\alpha_k}{\alpha_k'} = \frac{\alpha_\ell}{\alpha_\ell'}.
\]

Equivalently, the scaling factors $a_k$ and $a_\ell$ are equal. Since the expert graph is connected, this argument propagates along any path, implying that there exists a single constant $a \neq 0$ such that
\[
a_k = a \quad \text{for all experts } k \text{ with } \alpha_k \neq 0.
\]

Consequently, for any comparison $(\tau_i,\tau_j)$ labeled by an expert with nonzero trust, the transformed reward differences satisfy
\[
R'(\tau_i) - R'(\tau_j) = a\bigl(R(\tau_i) - R(\tau_j)\bigr).
\]

In particular, this relation holds for every edge $\{\tau_i,\tau_j\}$ in the trajectory graph $G$ defined by the global trajectory connectedness assumption.

\paragraph{Step 4: Global affine relation for $R$.}
Because $G$ is connected, fix an arbitrary reference trajectory $\tau_0\in V$. For any other trajectory $\tau\in V$, choose a path $(\tau_0,\tau_1),\ldots,(\tau_{m-1},\tau_m=\tau)$ in $G$. Summing the edge-wise relation along this path yields
\[
R'(\tau) - R'(\tau_0)
= \sum_{i=0}^{m-1} \big(R'(\tau_{i+1}) - R'(\tau_i)\big)
= a \sum_{i=0}^{m-1} \big(R(\tau_{i+1}) - R(\tau_i)\big)
= a\big(R(\tau) - R(\tau_0)\big).
\]
Defining $b := R'(\tau_0) - a R(\tau_0)$, we obtain
\[
R'(\tau) = a\,R(\tau) + b \quad \text{for all trajectories } \tau.
\]
Thus, any two parameter sets that agree on all induced probabilities must have rewards related by a global affine transformation with slope $a \neq 0$.

\paragraph{Step 5: Induced constraints on trust parameters.}
Substituting $R'(\tau) = a R(\tau) + b$ into \eqref{eq:prod-equality} for an arbitrary expert $k$ and any observed pair $(\tau_i,\tau_j)$ with $R(\tau_i) \neq R(\tau_j)$ gives
\[
\alpha_k (R(\tau_i)-R(\tau_j))
 =
\alpha_k' \big(a R(\tau_i) + b - (a R(\tau_j) + b)\big)
 =
\alpha_k' a (R(\tau_i)-R(\tau_j)).
\]
By expert connectedness assumption, every expert $k$ with $\alpha_k \neq 0$ participates in at least one comparison $(\tau_i,\tau_j) \in D^{(k)}$ for which $R(\tau_i) \neq R(\tau_j)$. For such a comparison, we can cancel the nonzero reward difference $R(\tau_i) - R(\tau_j)$ and conclude that
\[
\alpha_k = a\,\alpha_k' \;\Longrightarrow\; \alpha_k' = \alpha_k / a
\]
for all experts $k$ with $\alpha_k \neq 0$. Experts with $\alpha_k = 0$ are trivially matched by setting $\alpha_k' = 0$. Thus, any two parameter settings that induce the same likelihood must be related by the stated affine reparameterization.

\paragraph{Step 6: Completeness of the equivalence class.}
Conversely, if we take any $a \neq 0$ and $b \in \mathbb{R}$ and define
\[
R'(\tau) := a\,R(\tau) + b, \qquad \alpha_k' := \alpha_k / a,
\]
then for every observed pair,
\[
\alpha_k' (R'(\tau_i) - R'(\tau_j))
=
\frac{\alpha_k}{a} \big( a R(\tau_i) + b - (a R(\tau_j) + b) \big)
=
\alpha_k (R(\tau_i) - R(\tau_j)),
\]
so the induced probabilities and likelihoods are identical. Therefore the affine family described above is exactly the equivalence class of parameters compatible with the observed data.

This establishes that the model is identifiable up to a global affine transformation of $R$ and the corresponding common scaling of $\{\alpha_k\}$.
\end{proof}

\begin{corollary}
\label{cor:maxnorm-ident}
\textbf{Identifiability under max-normalized bounded trust}

Define the effective trust parameters used inside the sigmoid by
\[
\tilde\alpha_k := \tanh(\alpha_k),\quad 
\bar\alpha_k := \frac{\tilde\alpha_k}{\max_{k'}|\tilde\alpha_{k'}|},
\]
with the convention that the normalization is only applied when the denominator is nonzero. Consider the (effective)
logistic preference model
\[
\Pr\!\left(y_{ij}^{(k)} = 1 \mid R, \bar\alpha_k \right)
:= \sigma\!\left( \bar\alpha_k \left( R(\tau_i) - R(\tau_j) \right) \right),
\]
and assume the identifiability conditions of Theorem~\ref{thm:identifiability}. Assume also that
\[
\max_k |\tilde\alpha_k| > 0 \quad (\text{so that } \bar\alpha \text{ is well-defined}).
\]
Then the parameters $(R,\bar\alpha)$ are identifiable up to the reduced equivalence class
\[
R'(\tau) = \pm R(\tau) + b,
\qquad
\bar\alpha'_k = \pm \bar\alpha_k \quad \forall k,
\]
for some constant $b \in \mathbb{R}$. In particular, the continuous scaling ambiguity from
Theorem~\ref{thm:identifiability} disappears.
\end{corollary}

\begin{proof}
Let $(R,\bar\alpha)$ and $(R',\bar\alpha')$ induce the same probabilities for all observed comparisons:
\[
\sigma\!\left( \bar\alpha_k \left( R(\tau_i) - R(\tau_j) \right) \right)
=
\sigma\!\left( \bar\alpha'_k \left( R'(\tau_i) - R'(\tau_j) \right) \right)
\quad \forall (i,j,k)\ \text{observed.}
\]

\noindent\textbf{Step 1 (invert sigmoid).}
By strict monotonicity of $\sigma(\cdot)$ (Lemma~\ref{lem:logistic-ident}), equality of probabilities implies equality of
the logits, hence for every observed $(i,j,k)$,
\[
\bar\alpha_k \left( R(\tau_i) - R(\tau_j) \right)
=
\bar\alpha'_k \left( R'(\tau_i) - R'(\tau_j) \right).
\]

\noindent\textbf{Step 2 (Theorem~\ref{thm:identifiability} structure).}
Under the overlap and connectedness assumptions of Theorem~\ref{thm:identifiability}, the same argument as in its proof
yields the existence of constants $a\neq 0$ and $b\in\mathbb{R}$ such that
\[
R'(\tau) = aR(\tau) + b,
\qquad
\bar\alpha'_k = \frac{\bar\alpha_k}{a}
\quad \forall k \text{ with } \bar\alpha_k \neq 0.
\]
(Exactly as in Theorem~\ref{thm:identifiability}, with $\bar\alpha$ in place of $\alpha$.)

\noindent\textbf{Step 3 (use max-normalization).}
By construction of the max-normalization, $\|\bar\alpha\|_\infty=\max_k|\bar\alpha_k|=1$, and similarly
$\|\bar\alpha'\|_\infty=1$, whenever the denominator is nonzero. Taking $\ell_\infty$ norms in
$\bar\alpha'=\bar\alpha/a$ gives
\[
1=\|\bar\alpha'\|_\infty=\left\|\frac{\bar\alpha}{a}\right\|_\infty=\frac{1}{|a|}\|\bar\alpha\|_\infty=\frac{1}{|a|},
\]
so $|a|=1$, i.e., $a\in\{+1,-1\}$. Therefore
\[
R'(\tau)=\pm R(\tau)+b,
\qquad
\bar\alpha'=\pm \bar\alpha,
\]
which proves the corollary.
\end{proof}

\refstepcounter{section}
\section*{Appendix~\thesection: Trust Gradient Analysis}
\label{app:appendix_b}
\label{app:noisy-expert-gradient}
\addcontentsline{toc}{section}{Appendix~\thesection: Trust Gradient Analysis}

This appendix gives the derivation summarized in Section~\ref{sec:alpha-analysis}. We first analyze the raw trust gradient of Eq.~\eqref{eq:alpha-grad} for reliable, noisy, and adversarial experts under a fixed reward, and then treat the noisy case exactly for the trust the algorithm actually updates.

\subsection*{Case analysis for raw trust}

\paragraph{Reliable experts (\(\alpha_k\) increases).}
When expert \(k\)'s preferences align with the reward differences, for instance when \(y^{(k)}_{ij} = 1\) (expert prefers \(\tau_i\)) and \(\Delta R_{ij} > 0\), we have \(\sigma(\alpha_k \Delta R_{ij}) \approx 1\), leading to:
\[
(\sigma(\alpha_k \Delta R_{ij}) - y^{(k)}_{ij}) \cdot \Delta R_{ij} \approx (\text{positive number less than one} - 1) \cdot (\text{positive}) < 0.
\]
Since gradient descent updates \(\alpha_k \leftarrow \alpha_k - \eta \frac{\partial \mathcal{L}}{\partial \alpha_k}\), this negative gradient causes \(\alpha_k\) to increase. Reliable experts consistently push \(\alpha_k\) upward, increasing their influence on the learned reward.

\paragraph{Noisy experts (zero expected gradient at zero trust).}
A noisy expert is approximately uncorrelated with the sign of the reward difference \(\Delta R_{ij}\): conditioned on \(\Delta R_{ij}\), the expert reports either preference with roughly equal probability. Near the start of training, \(\sigma(\alpha_k \Delta R_{ij}) \approx \tfrac{1}{2}\) for every expert, reliable, noisy, or adversarial alike, because we explicitly initialize the reward model at \(R\equiv 0\) (Section~\ref{sec:cross-entropy-weight}), so \(\Delta R_{ij}\approx 0\) at the start of training. With \(\sigma(\alpha_k \Delta R_{ij}) \approx \tfrac{1}{2}\), the gradient with respect to \(\alpha_k\) becomes
\[
\frac{\partial \mathcal{L}}{\partial \alpha_k}
\;\approx\;
\mathbb{E}_{(\Delta R_{ij},\, y^{(k)}_{ij}) \sim \mathcal{D}_k}
\bigl[(\tfrac{1}{2} - y^{(k)}_{ij}) \, \Delta R_{ij}\bigr].
\]
To make the cancellation explicit, apply the law of total expectation by first conditioning on \(\Delta R_{ij}\):
\[
\mathbb{E}_{(\Delta, y)}\!\left[(\tfrac{1}{2} - y)\Delta\right]
=
\mathbb{E}_{\Delta}\!\left[
\Delta \cdot \mathbb{E}_{y \mid \Delta}[\tfrac{1}{2} - y \mid \Delta]
\right].
\]
Let
\[
p_k(\Delta) := \Pr(y = 1 \mid \Delta)
\]
denote expert \(k\)’s conditional probability of preferring the first trajectory given a reward gap \(\Delta\). Since \(y \in \{0,1\}\), we have \(\mathbb{E}[y \mid \Delta] = p_k(\Delta)\), leading to:
\[
\mathbb{E}_{(\Delta, y)}\!\left[(\tfrac{1}{2} - y)\Delta\right]
=
\mathbb{E}_{\Delta}\!\left[
\Delta \, (\tfrac{1}{2} - p_k(\Delta))
\right].
\]
For a noisy expert, labels are essentially random given \(\Delta\), so \(p_k(\Delta) \approx \tfrac{1}{2}\) for most \(\Delta\) under \(\Delta \sim \mathcal{D}_k\). Consequently, the integrand is approximately zero, and the expectation vanishes:
\[
\frac{\partial \mathcal{L}}{\partial \alpha_k}
\approx 0.
\]
Thus, at \(\alpha_k=0\) there is no consistent first-order signal pushing a genuinely noisy expert's trust away from zero. Because the cancellation relies only on the labels being uncorrelated with \(\Delta R_{ij}\), it continues to hold as \(\Delta R_{ij}\) grows. The next subsection shows that at nonzero trust the expected update points back toward zero.

\paragraph{Adversarial experts (\(\alpha_k\) becomes negative).}
When expert \(k\) systematically prefers lower-reward trajectories, for example when \(y^{(k)}_{ij} = 1\) but \(\Delta R_{ij} < 0\), we obtain \(\sigma(\alpha_k \Delta R_{ij}) \approx \sigma(\text{negative}) < 1\). Since \(\Delta R_{ij} < 0\), this leads to:
\[
(\sigma(\alpha_k \Delta R_{ij}) - y^{(k)}_{ij}) \cdot \Delta R_{ij} \approx (\text{positive number less than one} - 1) \cdot (\text{negative}) > 0,
\]
pushing \(\alpha_k\) downward into the negative regime. When \(\alpha_k < 0\), the model effectively flips the expert's preferences, transforming adversarial feedback into useful signal aligned with the global reward.  

\subsection*{Noisy expert under bounded, max-normalized trust}

We now state the noisy case exactly for the bounded, max-normalized $\bar\alpha_k$ of Section~\ref{sec:cross-entropy-weight}, trained through the trust path $\mathcal{L}_A$ of Eq.~\eqref{eq:loss-A}. The claim is that the expected trust update for a noisy expert is zero at $\alpha_k=0$ and points toward zero at every $\alpha_k\neq0$.

\paragraph{Setup.}
Write $\tilde\alpha_k=\tanh(\alpha_k)$ and
\[
\bar\alpha_k=\frac{\tilde\alpha_k}{c},
\qquad
c:=\max_{k'}|\tilde\alpha_{k'}|\in(0,1],
\]
with $c$ detached from backpropagation and, by the convention of Corollary~\ref{cor:maxnorm-ident}, applied only when $c\neq0$. In $\mathcal{L}_A$ the reward gap is frozen, so we condition on $\Delta:=\mathrm{sg}(\Delta R_{ij})$ and write $p_k(\Delta):=\Pr(y^{(k)}=1\mid\Delta)$ for expert $k$'s conditional preference probability. Differentiating through $\bar\alpha_k$ gives the population update
\begin{equation}
\label{eq:pop-grad}
\mathbb{E}\!\left[\frac{\partial \mathcal{L}_A}{\partial \alpha_k}\right]
=
\underbrace{\frac{1-\tilde\alpha_k^{2}}{c}}_{=:\;\kappa(\alpha_k)}
\;
\underbrace{\mathbb{E}_{\Delta}\Big[\Delta\,\big(\sigma(\bar\alpha_k \Delta) - p_k(\Delta)\big)\Big]}_{\text{gap-weighted prediction error}}.
\end{equation}
The bracket decides the \emph{direction} of the update and is where expert type enters; the prefactor $\kappa(\alpha_k)=\tanh'(\alpha_k)/c$ is strictly positive for every finite $\alpha_k$, so bounding and normalizing can rescale a step but never reverse it (raw trust is the special case $\kappa\equiv1$, $\bar\alpha_k=\alpha_k$, Eq.~\eqref{eq:alpha-grad}). A noisy expert is defined by $p_k(\Delta)=\tfrac{1}{2}$ for all $\Delta$, so \eqref{eq:pop-grad} becomes
\begin{equation}
\label{eq:pop-grad-noisy}
\mathbb{E}\!\left[\frac{\partial \mathcal{L}_A}{\partial \alpha_k}\right]
=
\kappa(\alpha_k)\,\mathbb{E}_{\Delta}\Big[\Delta\,\big(\sigma(\bar\alpha_k \Delta) - \tfrac{1}{2}\big)\Big].
\end{equation}

\paragraph{Step 1: zero trust is a fixed point.}
If $\alpha_k=0$ then $\bar\alpha_k=0$ for any $c>0$, so $\sigma(\bar\alpha_k\Delta)=\tfrac12$ on every pair. The bracket in \eqref{eq:pop-grad-noisy} vanishes term by term and the update is exactly zero, whatever the other experts are doing.

\paragraph{Step 2: every nonzero trust is pushed back toward zero.}
Fix $\alpha_k\neq0$ and a pair with $\Delta\neq0$. Since $\sigma$ is strictly increasing, $\sigma(z)-\tfrac12$ has the sign of $z$, so
\[
\operatorname{sign}\Big(\Delta\big(\sigma(\bar\alpha_k\Delta)-\tfrac12\big)\Big)
=\operatorname{sign}\big(\bar\alpha_k\Delta^{2}\big)
=\operatorname{sign}(\alpha_k).
\]
The sign is fixed \emph{pointwise}, on every gap rather than on average, so averaging preserves it and $\kappa>0$ cannot flip it: $\operatorname{sign}\mathbb{E}[\partial\mathcal{L}_A/\partial\alpha_k]=\operatorname{sign}(\alpha_k)$ for every $\alpha_k\neq0$. Gradient descent therefore subtracts an expected gradient with the same sign as $\alpha_k$, so the expected update points toward zero from any starting value. With Step~1 this proves the claim, for any admissible $c$ and with no assumption on the distribution of $\Delta$ beyond nondegeneracy.

\makeatletter
\setlength{\@fptop}{0pt}
\setlength{\@fpsep}{1.5\floatsep}
\setlength{\@fpbot}{0pt plus 1fil}
\makeatother

\refstepcounter{section}
\section*{Appendix~\thesection: Additional Synthetic Diagnostics}
\label{app:synthetic-sensitivity}
\addcontentsline{toc}{section}{Appendix~\thesection: Additional Synthetic Diagnostics}

This appendix gives the full sensitivity sweeps summarized in Section~\ref{sec:synthetic-sensitivity} and the trust learning curves for the partial-adversary study. All settings follow Section~\ref{sec:synthetic-branch} unless stated.

\paragraph{Trajectory length \(T\).}
Because \(R(\tau)=\sum_{t=1}^{T}\tanh(\theta^\top s_t)\), the reward gradient accumulates over \(T\) steps. Holding \(\eta_\theta\) fixed while changing \(T\) therefore rescales the effective update on \(\Delta R\). To isolate length from that confound, we repeat the branch diagnostic over \(T\in\{10,25,50,100,200\}\) with
\[
\eta_\theta(T)=0.05\cdot(50/T),
\]
so the sum-over-\(T\) step size stays matched to the paper default at \(T{=}50\) (init \(\mathrm{rms}|\Delta R|_0\) is still calibrated per \(T\); \(120\) seeds; \(n{=}500\)). Figure~\ref{fig:synthetic-T-sweep} shows that the qualitative ranking is stable across mixtures~(a) to~(c): under a reliable majority~\textbf{(a)},~\textbf{(b)}, Stabilized remains perfect across \(T\) while Standard stays weak; under 1R3A~\textbf{(c)}, Stabilized follows the adversary majority and Standard is mixed. Thus the \(T{=}50\) conclusions are not an artifact of a single horizon once \(\eta_\theta\) is scaled with \(T\).

\begin{figure}[!htbp]
    \centering
    \begin{subfigure}[b]{0.32\linewidth}
        \centering
        \includegraphics[width=\linewidth]{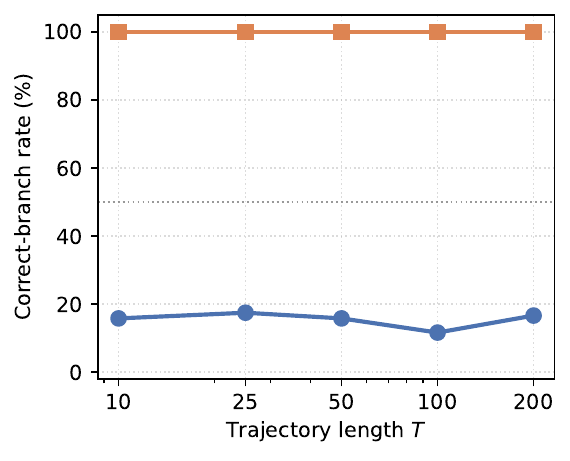}
        \caption{3R1N}
        \label{fig:synthetic-T-3R1N}
    \end{subfigure}\hfill
    \begin{subfigure}[b]{0.32\linewidth}
        \centering
        \includegraphics[width=\linewidth]{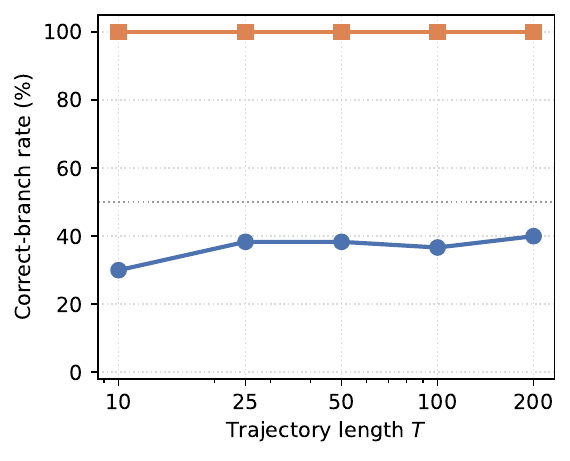}
        \caption{3R1A}
        \label{fig:synthetic-T-3R1A}
    \end{subfigure}\hfill
    \begin{subfigure}[b]{0.32\linewidth}
        \centering
        \includegraphics[width=\linewidth]{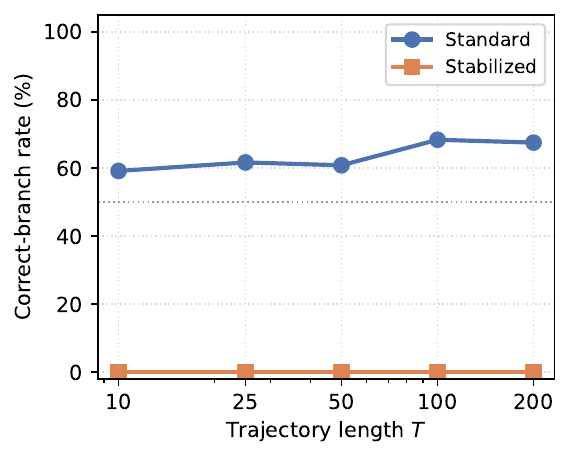}
        \caption{1R3A}
        \label{fig:synthetic-T-1R3A}
    \end{subfigure}
    \caption{Correct-branch rate vs.\ trajectory length \(T\) under the shared head with \(\eta_\theta=0.05\cdot(50/T)\), over \(120\) seeds, for \(n=500\).}
    \label{fig:synthetic-T-sweep}
\end{figure}

\paragraph{Initialization scale.}
We study the sensitivity of branch selection to the initialization scale, measured as \(\mathrm{rms}|\Delta R|_0\) at step \(0\) and controlled by rescaling \(\|\theta\|\). We sweep
\[
\mathrm{rms}|\Delta R|_0\in\{0,\,0.05,\,0.1,\,0.25,\,0.5,\,1.4,\,6\}
\]
over \(200\) seeds for 3R1A, 2R1A1N, 3R1N, and 1R3A (\(n{=}500\)). Moderate gaps (\(\approx 0.25\) to \(0.5\)) are closest to a default MLP init after a \(T{=}50\) sum.

Table~\ref{tab:synthetic-sigma-consensus} and Figure~\ref{fig:synthetic-sigma-consensus} report correct-branch rates under the shared-head recipe. Small-to-moderate inits (\(\mathrm{rms}|\Delta R|_0\le 0.5\)) keep 3R1A~\textbf{(a)}, 2R1A1N~\textbf{(b)}, and 3R1N~\textbf{(c)} near-perfect; the joint band where \emph{both} 3R1N and 3R1A exceed \(90\%\) covers \(\mathrm{rms}\in\{0,0.05,0.1,0.25,0.5\}\). At the Standard scale (\(\mathrm{rms}{=}1.4\)) and beyond, correct-branch rates collapse on the reliable-majority mixes, while 1R3A~\textbf{(d)} flips toward the correct branch as the large init fights the adversary majority. Overall, initialization scale is a first-order failure mode alongside the reliable-versus-adversarial count.

\begin{figure}[!htbp]
    \centering
    \begin{minipage}{\linewidth}
    \centering
\captionof{table}{Correct-branch rate (\(\mathrm{corr}(\hat R,R^*)>0.5\)) for max-norm TTP vs.\ initialization scale under the shared head, over \(200\) seeds, for \(T=50\), \(n=500\).}
\label{tab:synthetic-sigma-consensus}
\setlength{\tabcolsep}{3.0pt}
\begin{tabular}{lccccccc}
\toprule
Mix. & \(0\) & \(0.05\) & \(0.1\) & \(0.25\) & \(0.5\) & \(1.4\) & \(6\) \\
\midrule
3R1A & \(100\) & \(100\) & \(100\) & \(100\) & \(98.5\) & \(34.5\) & \(16.0\) \\
2R1A1N & \(100\) & \(100\) & \(100\) & \(98.0\) & \(92.0\) & \(34.5\) & \(28.5\) \\
3R1N & \(100\) & \(100\) & \(100\) & \(100\) & \(100\) & \(13.0\) & \(0.0\) \\
1R3A & \(0.5\) & \(0.0\) & \(0.0\) & \(0.5\) & \(5.0\) & \(60.5\) & \(81.5\) \\
\bottomrule
\end{tabular}
    \end{minipage}

    \vspace{0.8em}
    \setcounter{subfigure}{0}

    \begin{subfigure}[b]{0.24\linewidth}
        \centering
        \includegraphics[width=\linewidth]{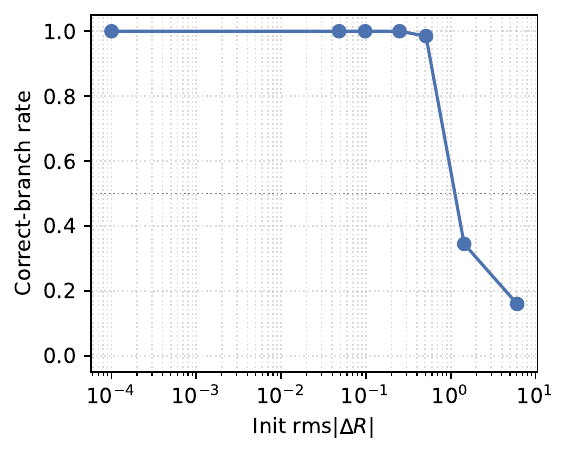}
        \caption{3R1A}
        \label{fig:synthetic-sigma-3R1A}
    \end{subfigure}\hfill
    \begin{subfigure}[b]{0.24\linewidth}
        \centering
        \includegraphics[width=\linewidth]{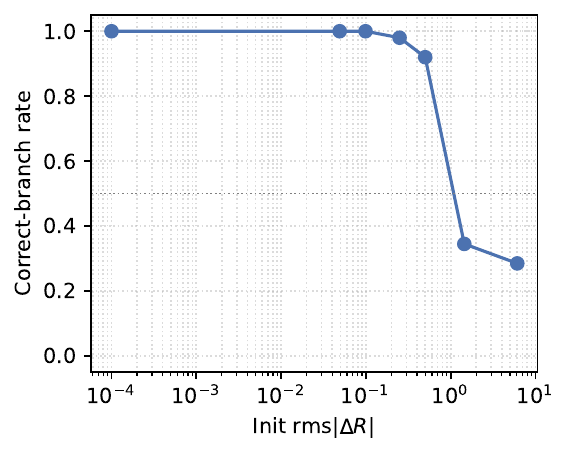}
        \caption{2R1A1N}
        \label{fig:synthetic-sigma-2R1A1N}
    \end{subfigure}\hfill
    \begin{subfigure}[b]{0.24\linewidth}
        \centering
        \includegraphics[width=\linewidth]{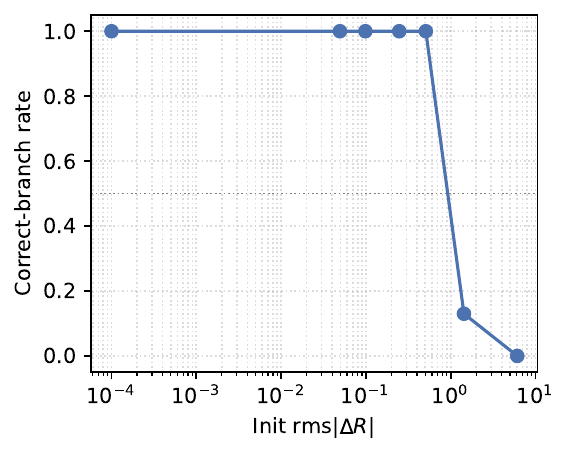}
        \caption{3R1N}
        \label{fig:synthetic-sigma-3R1N}
    \end{subfigure}\hfill
    \begin{subfigure}[b]{0.24\linewidth}
        \centering
        \includegraphics[width=\linewidth]{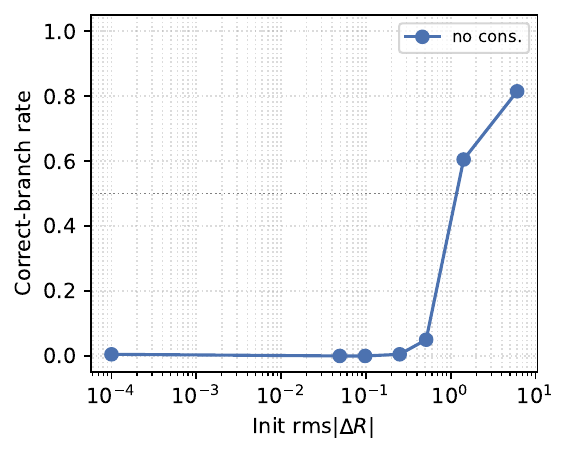}
        \caption{1R3A}
        \label{fig:synthetic-sigma-1R3A}
    \end{subfigure}
    \caption{Correct-branch rate vs.\ empirical initialization scale \(\mathrm{rms}|\Delta R|\) under the shared head, over \(200\) seeds, for \(T=50\), \(n=500\).}
    \label{fig:synthetic-sigma-consensus}
\end{figure}

\paragraph{Trust learning curves under partial adversaries.}
Figure~\ref{fig:partial-adversary-alpha-curves} shows the corresponding Stabilized \(\bar\alpha_k\) learning curves over training: in every case \(\bar\alpha_A\) reaches its steady-state level within the first few dozen of the \(400\) steps, so it is the level it settles at, not the speed of convergence, that tracks corruption strength. The perfect-flip and \(\beta\in\{-0.5,-0.25\}\) cases settle directly at their respective negative values; under \(p{=}0.5\) stochastic flips, \(\bar\alpha_A\) settles essentially at zero; and under \(p{=}0.25\), the curve briefly dips negative before crossing zero and settling at a mildly positive value, the one condition among the five where the sign of \(\bar\alpha_A\) is not fixed throughout training.

\begin{figure}[!htbp]
    \centering
    \begin{subfigure}[b]{0.32\linewidth}
        \centering
        \includegraphics[width=\linewidth]{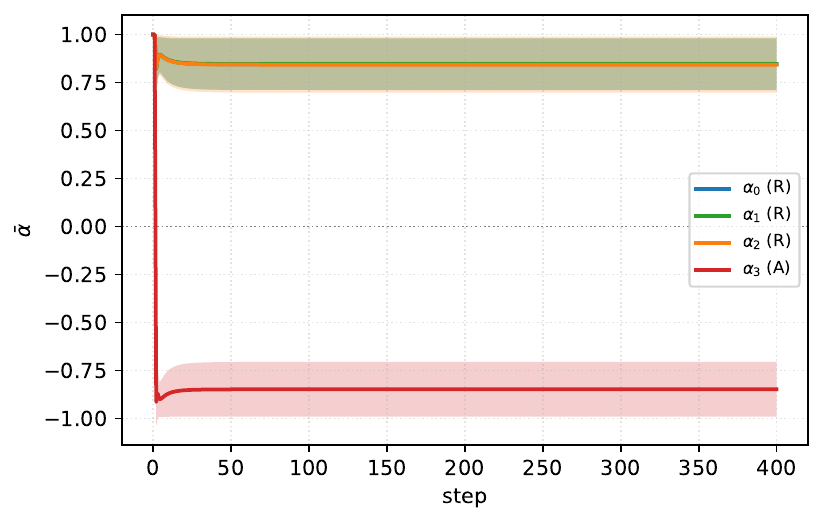}
        \caption{Perfect flip}
        \label{fig:partial-adversary-alpha-perfect}
    \end{subfigure}\hfill
    \begin{subfigure}[b]{0.32\linewidth}
        \centering
        \includegraphics[width=\linewidth]{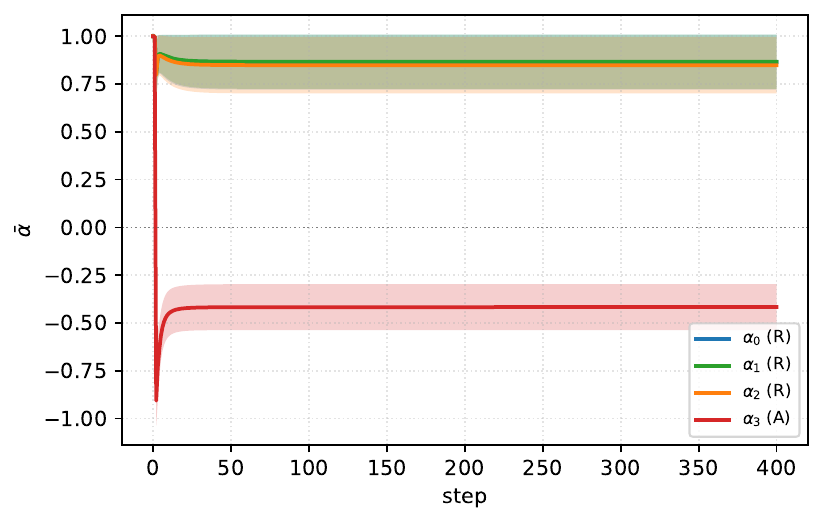}
        \caption{$\beta{=}-0.5$}
        \label{fig:partial-adversary-alpha-beta05}
    \end{subfigure}\hfill
    \begin{subfigure}[b]{0.32\linewidth}
        \centering
        \includegraphics[width=\linewidth]{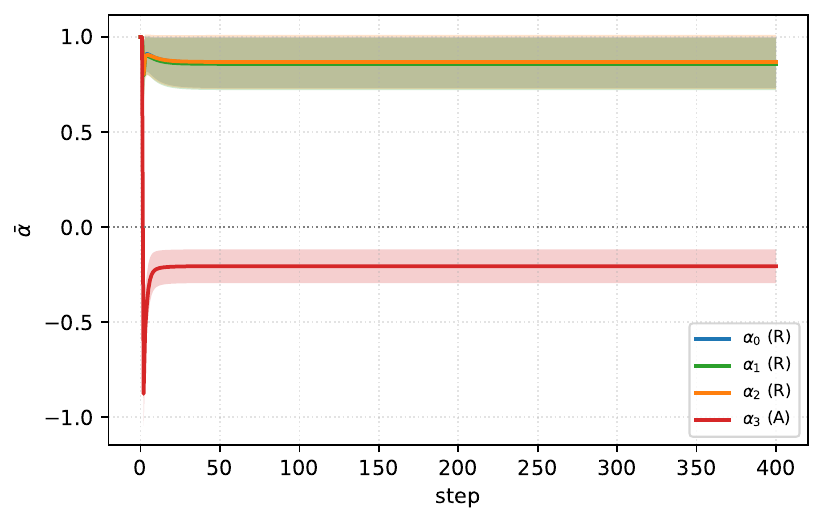}
        \caption{$\beta{=}-0.25$}
        \label{fig:partial-adversary-alpha-beta025}
    \end{subfigure}\\[0.6em]
    \begin{subfigure}[b]{0.32\linewidth}
        \centering
        \includegraphics[width=\linewidth]{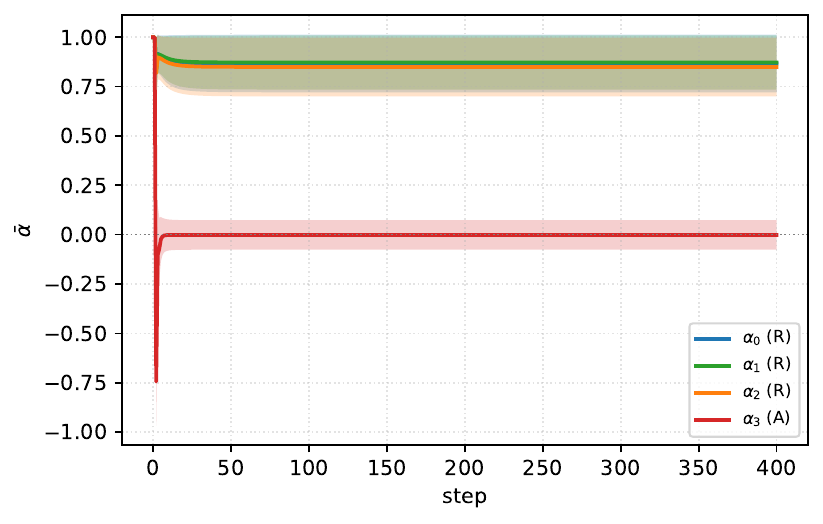}
        \caption{Stoch.\ $p{=}0.5$}
        \label{fig:partial-adversary-alpha-stoch05}
    \end{subfigure}\hspace{0.04\linewidth}
    \begin{subfigure}[b]{0.32\linewidth}
        \centering
        \includegraphics[width=\linewidth]{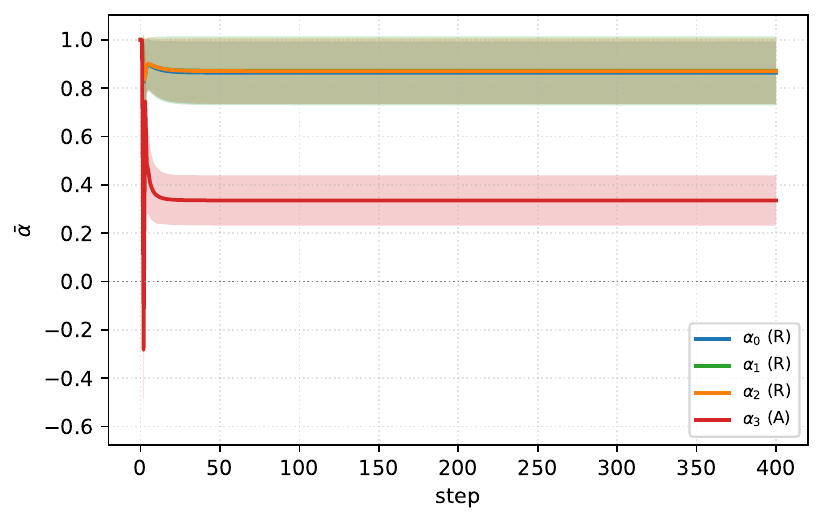}
        \caption{Stoch.\ $p{=}0.25$}
        \label{fig:partial-adversary-alpha-stoch025}
    \end{subfigure}
    \caption{Per-expert normalized trust \(\bar\alpha_k\) vs.\ training step under Stabilized on shared-head 3R1A (mean \(\pm\) std), over \(200\) seeds, for \(n=500\).}
    \label{fig:partial-adversary-alpha-curves}
\end{figure}

\end{document}